\documentclass[11pt,letter]{article}
\usepackage[papersize={8.5in,11in},margin=1in]{geometry}
\usepackage[utf8]{inputenc} 
\usepackage[T1]{fontenc}    
\usepackage{hyperref}       
\usepackage{url}            
\usepackage{booktabs}       
\usepackage{amsfonts}       
\usepackage{nicefrac}       
\usepackage{microtype}      

\usepackage{amssymb}
\usepackage{times}
\usepackage{helvet}
\usepackage{courier}
\newtheorem{assumption}{Assumption}
\newtheorem{theorem}{Theorem}
\newtheorem{corollary}{Corollary}
\newtheorem{proof}{Proof}

\usepackage{algorithm}
\usepackage{algorithmic}
\usepackage{amsmath}
\usepackage{graphicx}
\usepackage{caption}
\usepackage{subcaption}
\usepackage{epsfig}
\usepackage{epstopdf}
\usepackage{multirow}

\title{Asynchronous Stochastic Gradient Descent with Variance Reduction for Non-Convex Optimization}
\begin{document}

\author{Zhouyuan Huo\\
zhouyuan.huo@mavs.uta.edu \\
\and
Heng Huang\\
heng@uta.edu
}

\maketitle

\begin{abstract}

We provide the first theoretical analysis on the convergence rate of asynchronous stochastic gradient descent with variance reduction (AsySVRG) for
non-convex optimization. Asynchronous stochastic gradient descent (AsySGD) has been broadly used in solving neural network and it is proved to converge with $O(1/\sqrt{T})$.
Recent studies have shown that  asynchronous SGD method with variance reduction technique converges
with a linear convergence rate on convex problem. However, there is no work to analyze asynchronous SGD with variance reduction technique on non-convex problem.
In this paper, we consider two asynchronous parallel implementations of SVRG: one is on distributed-memory architecture and the other is on shared-memory architecture.
We prove that both methods can converge with a rate of $O(1/T)$, and a linear speedup is achievable when we increase the number of workers.
Experimental results on neural network with real data (MNIST and CIFAR-10) also demonstrate our statements.

\end{abstract}

\section{Introduction}
With the boom of data, training machine learning models with large datasets is a critical problem. Researchers extend batch gradient descent (GD) method to stochastic gradient
descent (SGD) method or mini-batch gradient descent method to relieve the complexity of computation in each iteration and reduce the total time complexity of optimization
\cite{bottou2010large}. However, when data is very large, serial algorithm is
time-consuming. Asynchronous parallelism has been successfully applied to speed up many state-of-the-art optimization algorithms
\cite{recht2011hogwild,liu2014asynchronous,lian2015asynchronous,zhang2014asynchronous} because there is no need of synchronization
between workers. Two different types of parallelism have been widely researched, one is distributed-memory parallelism on multiple machines
\cite{agarwal2011distributed,lian2015asynchronous,zhang2014asynchronous,dean2012large,zhang2015fast,huo2016distributed,huo2016decoupled}  and the other one  is shared-memory parallelism on a multi-core machine
\cite{recht2011hogwild,zhao2016fast,langford2009slow,,gu2016asynchronous}. Deep learning is a typical case where asynchronous SGD has gained great
success\cite{lecun2015deep,dean2012large,lian2015asynchronous,ngiam2011optimization}. Deep neural network always has large set of parameters and trains with large datasets.

 Due to efficiency, SGD method has been widely used to solve different kinds of machine learning models, both convex and non-convex.
 However, because we use stochastic gradient to approximate full gradient,  a decreasing learning rate has to be applied to guarantee
 convergence. Thus, SGD leads to a slow convergence rate $O(1/T)$ on strongly convex smooth problem and  $O(1/\sqrt T)$ on non-convex smooth problem. Recently, variance reduced SGD algorithms \cite{roux2012stochastic,johnson2013accelerating,defazio2014saga} have
gained many attentions to solve machine learning problem. These methods can reduce the variance of stochastic gradient during optimization and are proved to have linear convergence rate on strongly convex smooth problem. In
\cite{zhu2016svrg,reddi2016stochastic}, the stochastic variance reduced gradient (SVRG) method is analyzed on non-convex smooth problem, and a faster sublinear convergence rate $O(1/T)$ is proved to be achievable.

Although a faster convergence rate can be achieved by using variance reduction technique, sequential method on a single machine may still be not enough to solve large-scale problem efficiently.
Recently, asynchronous SVRG method has been implemented and studied on both distributed-memory architecture \cite{zhang2015fast} and shared-memory
architecture \cite{zhao2016fast}. It is proved that asynchronous SVRG method has linear convergence rate on strongly convex smooth problem. However, there is no theoretical analysis of asynchronous SVRG on non-convex problem yet.

In this paper, we focus on asynchronous SVRG method for non-convex optimization. Two different algorithms and analysis
are proposed in this paper on two different distributed architectures, one is shared-memory architecture and the other is distributed-memory architecture.
The key difference between these two categories lies on that distributed-memory architecture can ensure the atomicity of reading and writing the whole vector of $x$, while
the shared-memory architecture can usually just ensure atomic reading and writing on a single coordinate of $x$ \cite{lian2015asynchronous}. We implement asynchronous SVRG on two different architectures and analyze their
convergence rate. We prove that asynchronous SVRG can get an ergodic convergence rate $O(1/T)$ on both two different architectures. Besides, we also prove that a linear
speedup is achievable when we increase the number of workers.

We list our main contributions as follows:
\begin{itemize}
 \item We extend asynchronous shared-memory SVRG method to non-convex smooth problem. Our asynchronous SVRG on shared-memory architecture has faster convergence rate than ASYSG-INCON
 in \cite{lian2015asynchronous}. We prove that asynchronous SVRG has a convergence rate of $O(1/T)$ for non-convex optimization.

 \item We extend  asynchronous distributed-memory SVRG method to non-convex smooth problem. Our asynchronous SVRG on distributed-memory architecture has faster convergence rate than
 ASYSG-CON in \cite{lian2015asynchronous}. We prove that asynchronous SVRG has a convergence rate of $O(1/T)$
 for non-convex optimization.
\end{itemize}

\section{Notation}
In this paper, we consider the following non-convex finite-sum problem:
\begin{eqnarray}
\label{fs}
 \min\limits_{x \in \mathbb{R}^{d}} f(x) = \frac{1}{n}\sum\limits_{i=1}^n f_i(x)\,,
\end{eqnarray}
where $f(x)$ and $f_i(x)$ are just Lipschitz smooth.

Following \cite{lian2015asynchronous,reddi2016stochastic}, in non-convex optimization, we use the weighted average of the $\ell_2$ norm of all gradients $||\nabla f(x)||^2$ as metric to analyze its convergence property.
For further analysis, throughout this paper, we make the following assumptions for problem (\ref{fs}). All of them are very common assumptions in the theoretical analysis of stochastic
gradient descent method.

\begin{assumption}
We assume that following conditions hold,
\begin{itemize}
 \item \textbf{Independence:} All random samples $i$ are selected independently to each other.
 \item \textbf{Unbiased Gradient:}  The stochastic gradient $\nabla f_{i}(x)$ is unbiased,
 \begin{eqnarray}
  \mathbb{E} \left[ \nabla f_{i}(x) \right] = \nabla f(x)
 \end{eqnarray}
 \item \textbf{Lipschitz Gradient:} We say $f(x)$ is $L$-$smooth$ if there is a constant $L$ such that
 \begin{eqnarray}
  ||\nabla f(x) - \nabla f(y) || \leq L || x - y||
 \end{eqnarray}
 Throughout, we also assume that the function $f_{i}(x)$ is $L$-$smooth$, so that
$||\nabla f_{i}(x) - \nabla f_{i}(y) || \leq L || x - y||$

 \item \textbf{Bounded Delay:} Time delay variable $\tau$ is upper bounded, namely $\max \tau \leq \Delta$. In practice, $\Delta$ is related with the number of workers.

\end{itemize}
\end{assumption}

\section{Asynchronous Stochastic Gradient Descent with Variance Reduction for Shared-memory Architecture}
In this section, we propose  asynchronous SVRG method for shared-memory architecture, and prove that it converges with rate $O(1/T)$. In \cite{reddi2016stochastic,zhu2016svrg}, it is proved that SVRG has a convergence rate of $O(1/T)$ on non-convex problem. In this section, we follow the convergence analysis in \cite{reddi2016stochastic}, and extends it to asynchronous convergence analysis on shared-memory architecture.

\subsection{Algorithm Description}

Following the setting in \cite{lian2015asynchronous}, we define one iteration as a modification on any single component of $x$ in the shared memory.
We use $x_{t}^{s+1}$ to denote the value of parameter $x$ in the shared-memory after $(ms+t)$ iterations, and Equation (\ref{update_shared}) represents the update rule of
parameter $x$ in iteration $t$,
\begin{eqnarray}
\label{update_shared}
(x_{t+1}^{s+1})_{k_t} = (x_{t}^{s+1})_{k_t} - \eta (v_{t}^{s+1})_{k_t}\,,
\end{eqnarray}
where $k_t \in \{1,...,d\}$ is the index of component in $x$, and learning rate $\eta$ is constant. $v_t^{s+1}$ is defined as follows,
\begin{eqnarray}
\label{shared_v_2}
v_t^{s+1} = \frac{1}{|I_t|} \sum\limits_{i_t \in I_t} \left( \nabla f_{i_t}(\hat x_{t,i_t}^{s+1}) - \nabla f_{i_t}(\tilde{x}^s) + \nabla f(\tilde{x}^s) \right)
\end{eqnarray}
where $\tilde{x}^s$ denotes a snapshot of $x$ after every $m$ iterations. $\hat x_{t,i_t}^{s+1}$ denotes the parameter in a worker used to compute gradient with sample $i_t$ , $i_t$ denotes the
index of a  sample, and $I_t$ is index set of mini-batch samples. The definition of $\hat x_{t,i_t}^{s+1}$ follows the analysis in \cite{lian2015asynchronous},
where $\hat x_{t,i_t}^{s+1}$ is assumed to be some earlier state of $x$ in the
shared memory.
\begin{eqnarray}
\hat x_{t,i_t}^{s+1} = x_{t}^{s+1} - \sum\limits_{j \in J(t)} (x_{j+1}^{s+1} - x_j^{s+1})
\end{eqnarray}
where $J(t)\in \{t-1,....,t-\Delta \}$ is a subset
of index numbers in previous iterations, $\Delta$ is the upper bound of time delay.  In Algorithm \ref{alg2}, we summarize the asynchronous SVRG on shared-memory architecture.
\begin{algorithm}                      
\caption{Shared-AsySVRG}         
\label{alg2}                           
\begin{algorithmic}                    
  \STATE Initialize $x^0 \in \mathbb{R}^d$.\\
    \FOR{$s = 0,1,2,,..,S-1$ }
     \STATE $\tilde{x}^s \leftarrow  x^s$;
     \STATE Compute full gradient $\nabla f(\tilde{x}^s) \leftarrow  \frac{1}{n} \sum\limits_{i=1}^n \nabla f_{i}(\tilde{x}^s)$;
     \STATE \textbf{Parallel Computation on Multiple Threads}
     \FOR{$t=0,1,2,...,m-1$}
     \STATE Randomly select mini-batch $I_t$ from $\{1,....n \}$;
     \STATE Compute the gradient: $v_t^{s+1} \leftarrow \frac{1}{|I_t|} \sum\limits_{i_t \in I_t} \left( \nabla f_{i_t}(\hat x_{t,i_t}^{s+1}) - \nabla f_{i_t}(\tilde{x}^s) + \nabla f(\tilde{x}^s) \right)$
     \STATE Randomly select $k_t$ from $\{1,...,d \}$
      \STATE Update $(x_{t+1}^{s+1})_{k_t} \leftarrow  (x_{t}^{s+1})_{k_t} - \eta (v_{t}^{s+1})_{k_t}
$
     \ENDFOR
     \STATE $x^{s+1} \leftarrow x^{s+1}_{m}$
    \ENDFOR
\end{algorithmic}
\end{algorithm}

\subsection{Convergence Analysis}
\begin{corollary}
\label{lem4}
 For the definition of the variance reduced gradient $v_t^{s+1}$ in Equation (\ref{shared_v_2}), and as per \cite{reddi2015variance} we define,
\begin{eqnarray}
\label{shared_u_2}
u_t^{s+1} =  \frac{1}{|I_t|} \sum\limits_{i_t \in I_t} \left( \nabla f_{i_t}(x_{t}^{s+1}) - \nabla f_{i_t}(\tilde{x}^s) + \nabla f(\tilde{x}^s)\right)
\end{eqnarray}
We have the following inequality,
\begin{eqnarray}
\sum\limits_{t=0}^{m-1} \mathbb{E}\left[ ||v_{t}^{s+1}||^2 \right]  &\leq& \frac{2d}{d-2L^2 \Delta^2 \eta^2}  \sum\limits_{t=0}^{m-1}  \mathbb{E}\left[ ||u_t^{s+1} ||^2 \right]
\end{eqnarray}
\end{corollary}

where $\mathbb{E}\left[ ||u_{t}^{s+1}||^2 \right]$ is upper bounded in \cite{reddi2015variance}.
\begin{eqnarray}
\mathbb{E}\left[ ||u_{t}^{s+1}||^2 \right] &\leq& 2 \mathbb{E}\left[ ||  \nabla f(x_{t}^{s+1})||^2 \right]  +   \frac{2L^2}{b}\mathbb{E}\left[ || x_{t}^{s+1}- \tilde{x}^s ||^2 \right]
\end{eqnarray}

Then, we follow the  convergence proof of SVRG for non-convex optimization in \cite{reddi2016stochastic}, and extends it to asynchronous case.
\begin{theorem}
\label{thm_m_3}
 Let $c_m=0$, learning rate $\eta > 0$ is constant, $\beta_t = \beta >0$, $b$ denotes the size of mini-batch. We define:
\begin{eqnarray}
c_t = c_{t+1}(1+ \frac{\eta\beta_t}{d}+ \frac{4L^2\eta^2}{(d-2L^2 \Delta^2 \eta^2)b}) + \frac{4L^2}{(d-2L^2 \Delta^2 \eta^2)b}  (\frac{L^2 \Delta^2 \eta^3}{2d} + \frac{\eta^2L}{2})
\end{eqnarray}
\begin{eqnarray}
 \Gamma _t =  \frac{\eta}{2d} -\frac{4}{d-2L^2 \Delta^2 \eta^2}  (\frac{L^2 \Delta^2 \eta^3}{2d} + \frac{\eta^2L}{2} + c_{t+1}\eta^2)
\end{eqnarray}
such that
$\Gamma_t >0$  for  $0 \leq t \leq m-1$. Define  $\gamma = \min_t \Gamma_t$, $x^*$ is the optimal solution. Then, we have the following ergodic convergence rate for iteration $T$,
\begin{eqnarray}
\frac{1}{T}\sum\limits_{s=0}^{S-1}\sum\limits_{t=0}^{m-1} \mathbb{E}\left[  ||\nabla f(x_t^{s+1})||^2 \right]  \leq \frac{\mathbb{E}\left[  f( x^{0})  -  f( x^{*}) \right] }{T\gamma}
\end{eqnarray}
\end{theorem}

\begin{theorem}
\label{thm_m_4}
Let $\eta = \frac{u_0b}{Ln^\alpha}$, where $0 < u_0 < 1$ and $0< \alpha \leq1$, $\beta =  {2L}$,
$m = \lfloor  \frac{dn^{\alpha}}{6u_0b}  \rfloor $ and $T$ is total iteration. If time delay $\Delta$ is upper bounded by
\begin{eqnarray}
\Delta^2 < \min \{\frac{d}{2u_0b}, \frac{3d-28u_0bd}{28u_0^2b^2}  \}\end{eqnarray}

Then there exists universal constant $u_0$, $\sigma$, such that it holds that
$\gamma \geq \frac{\sigma b}{dLn^{\alpha}}$ in Theorem \ref{thm_m_3} and
\begin{eqnarray}
\frac{1}{T}\sum\limits_{s=0}^{S-1}\sum\limits_{t=0}^{m-1} \mathbb{E}\left[  ||\nabla f(x_t^{s+1})||^2 \right]  \leq \frac{dLn^{\alpha}\mathbb{E}\left[  f( x^{0})  -  f( x^{*}) \right] }{bT\sigma}
\end{eqnarray}
\end{theorem}
Since the convergence rate does not depend on the delay parameter $\Delta$, the negative effect of using old values of $x$ for stochastic gradient evaluation vanishes asymptoticly.
Thus, linear speedup is achievable if $\Delta^2$ is upper bounded.

\section{Asynchronous Stochastic Gradient Descent with Variance Reduction for Distributed-memory Architecture}
In this section, we propose  asynchronous SVRG algorithm for distributed-memory architecture, and prove that it converges with rate $O(1/T)$.

\subsection{Algorithm Description}
In each iteration, parameter $x$ is updated through the following update rule,
\begin{eqnarray}
x_{t+1}^{s+1} = x_{t}^{s+1} - \eta v_{t}^{s+1}
\end{eqnarray}
where learning rate $\eta$ is constant, $v_t^{s+1}$ represents the variance reduced gradient,
\begin{eqnarray}
\label{batch_v}
v_t^{s+1} =  \frac{1}{|I_t|} \sum\limits_{i_t \in I_t}  \left( \nabla f_{i_t}(x_{t-\tau_{i}}^{s+1}) - \nabla f_{i_t}(\tilde{x}^s) + \nabla f(\tilde{x}^s)\right)
\end{eqnarray}
where $\tilde{x}^s$ means a snapshot of $x$ after every $m$ iterations, and $x_{t-\tau_{i}}^{s+1}$ denotes the current parameter used to compute gradient in a worker. $i_t$
denotes the index of a sample, $\tau_i$ denotes time delay for each sample $i$, and mini-batch size is  $|I_t|$. We summarize the asynchronous SVRG on distributed-memory
architecture in the  Algorithm \ref{alg1_1} and Algorithm \ref{alg1_2}, Algorithm \ref{alg1_1} shows operations in server node, and Algorithm \ref{alg1_2} shows operations in
worker node.
\begin{algorithm}                      
\caption{Distributed-AsySVRG Server Node}          
\label{alg1_1}                           
\begin{algorithmic}                    
  \STATE Initialize $x^0 \in \mathbb{R}^d$.\\
    \FOR{$s = 0,1,2,,..,S-1$ }
     \STATE $\tilde{x}^s \leftarrow  x^s$;
     \STATE Broadcast $\tilde{x}^s$ to all workers.
     \STATE Receive gradient from all workers and compute full gradient $\nabla f(\tilde{x}^s) \leftarrow  \frac{1}{n} \sum\limits_{i=1}^n \nabla f_{i}(\tilde{x}^s)$;
     \FOR{$t=0,1,2,...,m-1$}
     \STATE Receive gradients  $\frac{1}{|I_t|} \sum\limits_{i_t \in I_t}  \nabla f_{i_t}(x_{t-\tau_{i}}^{s+1})  $ and $\frac{1}{|I_t|} \sum\limits_{i_t \in I_t}  \nabla f_{i_t}(\tilde{x}^s)  $ from a specific worker.
     \STATE Compute  gradient: $v_t^{s+1} \leftarrow  \frac{1}{|I_t|} \sum\limits_{i_t \in I_t}  \left( \nabla f_{i_t}(x_{t-\tau_{i}}^{s+1}) - \nabla f_{i_t}(\tilde{x}^s) + \nabla f(\tilde{x}^s)\right)
$
      \STATE Update $x_{t+1}^{s+1} \leftarrow  x_{t}^{s+1} - \eta v_{t}^{s+1}
$
     \ENDFOR
     \STATE $x^{s+1} \leftarrow x^{s}_{m}$
    \ENDFOR
\end{algorithmic}
\end{algorithm}

\begin{algorithm}                      
\caption{Distributed-AsySVRG Worker Node}          
\label{alg1_2}                           
\begin{algorithmic}      
  \STATE Receive parameter $x_{t-\tau_{i}}^{s+1}$ from server.
  \STATE Compute gradient  $\frac{1}{|I_t|} \sum\limits_{i_t \in I_t}  \nabla f_{i_t}(x_{t-\tau_{i}}^{s+1}) $ and send it to server.\\
  \STATE Compute gradient $\frac{1}{|I_t|} \sum\limits_{i_t \in I_t}  \nabla f_{i_t}(\tilde{x}^s)  $ and send it to server. \\
\end{algorithmic}
\end{algorithm}

\subsection{Convergence Analysis}

Our idea of convergence analysis and techniques come from \cite{reddi2016stochastic}, and we use it to analyze the convergence rate of distributed algorithm.
\begin{theorem}
\label{thm_m_1}
Let $c_m=0$, learning rate $ \eta > 0$ is constant, $\beta_t = \beta >0$, $b$ denotes the size of mini-batch. We define
\begin{eqnarray}
c_t =  c_{t+1}\left(1+\eta\beta_t + \frac{4L^2\eta^2}{(1-2L^2 \Delta^2 \eta^2)b}\right) + \frac{4L^2}{(1-2L^2 \Delta^2 \eta^2)b}  \left(\frac{L^2 \Delta^2 \eta^3}{2} + \frac{\eta^2L}{2}\right)
\end{eqnarray}
\begin{eqnarray}
 \Gamma _t = \frac{\eta}{2} -\frac{4}{(1-2L^2 \Delta^2 \eta^2)}  (\frac{L^2 \Delta^2 \eta^3}{2} + \frac{\eta^2L}{2} + c_{t+1}\eta^2)
\end{eqnarray}
such that
$\Gamma_t >0$  for  $0 \leq t \leq m-1$. Define  $\gamma = \min_t \Gamma_t$, $x^*$ is the optimal solution. Then, we have the following ergodic convergence rate for iteration $T$:
\begin{eqnarray}
\frac{1}{T}\sum\limits_{s=0}^{S-1}\sum\limits_{t=0}^{m-1} \mathbb{E}\left[  ||\nabla f(x_t^{s+1})||^2 \right]  \leq \frac{\mathbb{E}\left[  f( x^{0})  -  f( x^{*}) \right] }{T\gamma}
\end{eqnarray}
\end{theorem}

\begin{theorem}
\label{thm_m_2}
Suppose $f \in \mathcal{F}$. Let $\eta_t = \eta = \frac{u_0b}{Ln^\alpha}$, where $0 < u_0 < 1$ and $0< \alpha \leq1$, $\beta =  2{L}$,
$m = \lfloor  \frac{n^{\alpha}}{6u_0b}  \rfloor $ and $T$ is total iteration. If the time delay $\Delta$ is upper bounded by
\begin{eqnarray}
 \Delta^2 < \min \{ \frac{1}{2u_0b}, \frac{3 - 28u_0b}{28u_0^2b^2}  \}
\end{eqnarray}
then there exists universal constant $u_0$, $\sigma$, such that it holds that:
$\gamma \geq \frac{\sigma b}{Ln^{\alpha}}$ in Theorem \ref{thm_m_1} and
\begin{eqnarray}
\frac{1}{T}\sum\limits_{s=0}^{S-1}\sum\limits_{t=0}^{m-1} \mathbb{E}\left[  ||\nabla f(x_t^{s+1})||^2 \right]  \leq \frac{Ln^{\alpha}\mathbb{E}\left[  f( x^{0})  -  f( x^{*}) \right] }{bT\sigma}
\end{eqnarray}
\end{theorem}
Because the convergence rate has nothing to do with $\Delta$,  linear speedup is achievable.

\section{Experiments}
In this section, we perform experiments on distributed-memory architecture and shared-memory architecture respectively. One of the main purpose of our experiments is to validate the faster convergence
rate of asynchronous SVRG method, and the other purpose is to demonstrate its linear speedup property. The speedup we consider in this paper is
running time speedup when they reach similar performance, e.g. training loss function value. Given $T$ workers, running time speedup is defined as,
\begin{eqnarray}
 \text{Running time speedup of } T \text{ workers} = \frac{\text{Running time for the serial computation}}{\text{Running time of using } T \text{ workers}}
\end{eqnarray}

\subsection{Shared-memory Architecture}
We conduct experiment on a machine which has $2$ sockets, and each socket has $18$ cores. OpenMP library \footnote{https://openmp.org} is used to handle shared-memory parallelism,
We consider the multi-class problem on MNIST dataset \cite{lecun1998gradient}, and use $10,000$ training samples and $2,000$ testing samples in the experiment.
Each image sample is a vector of $784$ pixels.  We construct a toy three-layer neural network $(784 \times 100 \times 10)$, where ReLU activation function is used in the hidden layer
and there are $10$ classes in MNIST dataset. We train this neural network with softmax loss function, and $\ell_2$ regularization with weight $C = 10^{-3}$. We set mini-batch size
$ |I_t| = 10$, and inner iteration length $m=1,000$. Updating only one component in $x$ in each iteration is too time consuming, therefore we randomly select and update $1,000$
components in each iteration.

We compare following three methods in the experiment:
\begin{itemize}
 \item SGD: We implement stochastic gradient descent (SGD) algorithm and train with the best tuned learning rate. In our experiment, we use polynomial learning
 rate $\eta = \frac{\alpha}{(1+s)^{\beta}}$,
 where $\alpha$  denotes initial learning rate and we tune it from $\{1e^{-2},5e^{-2},1e^{-3},5e^{-3}, 1e^{-4},5e^{-4},1e^{-5},5e^{-5} \}$, $\beta$ is a variable in $\{0,0.1,...,1\}$ and $s$ denotes the epoch number.
 \item SVRG: We also implement stochastic gradient descent with variance reduction (SVRG) method and train with the best tuned constant learning rate $\alpha$.
  \item SGD-SVRG: SVRG method is sensitive to initial point, and  we apply SVRG on a pre-trained model using SGD.
  In the experiment,  we use the pre-trained model after $10$ iterations of SGD method.
\end{itemize}

We test three compared methods on MNIST dataset, and each method trained with best tuned learning rate. Figure \ref{ompcomp} shows the convergence rate of each method. We
compare three criterion in this experiment, loss function value on training dataset, training error rate, and testing error rate.  Figure
\ref{ompcomp_trainloss} shows the curves of loss function on training dataset, it is clear that SGD method converges faster than SVRG method in the first $20$ iterations, and after that, SVRG method outperforms SGD.
$\text{SGD}\_\text{SVRG}$ method  initializes with a pre-trained model, it has the best convergence rate. We are able to draw the same conclusion from
Figure \ref{ompcomp_trainerror} and Figure \ref{ompcomp_testerror}.

We also evaluate SVRG method with different number of threads, and Figure \ref{threads} presents the result of our experiment. In Figure \ref{threads_trainloss}, we plot curves for
each method when they get similar training loss value.  As we can see, the more threads we use in the computation, the less time we need to achieve a similar performance. This phenomenon
is reasonable, because iterations in a loop can be divided into multiple parts, and each thread handles one subset independently. The ideal result of parallel computation is linear
speedup, namely if we use $K$ threads, its working time should be $\frac{1}{K}$ of the time when we just use a single thread. Figure \ref{threads_speedup} shows the ideal speedup and
actual speedup in our experiment. We can find out that a almost linear speedup is achievable when we increase thread numbers. When the number of threads exceeds a threshold,
performance tends  to degrade.

\begin{figure}[h]
\centering
\begin{subfigure}[b]{0.32\textwidth}
\centering
\includegraphics[width=1.7in]{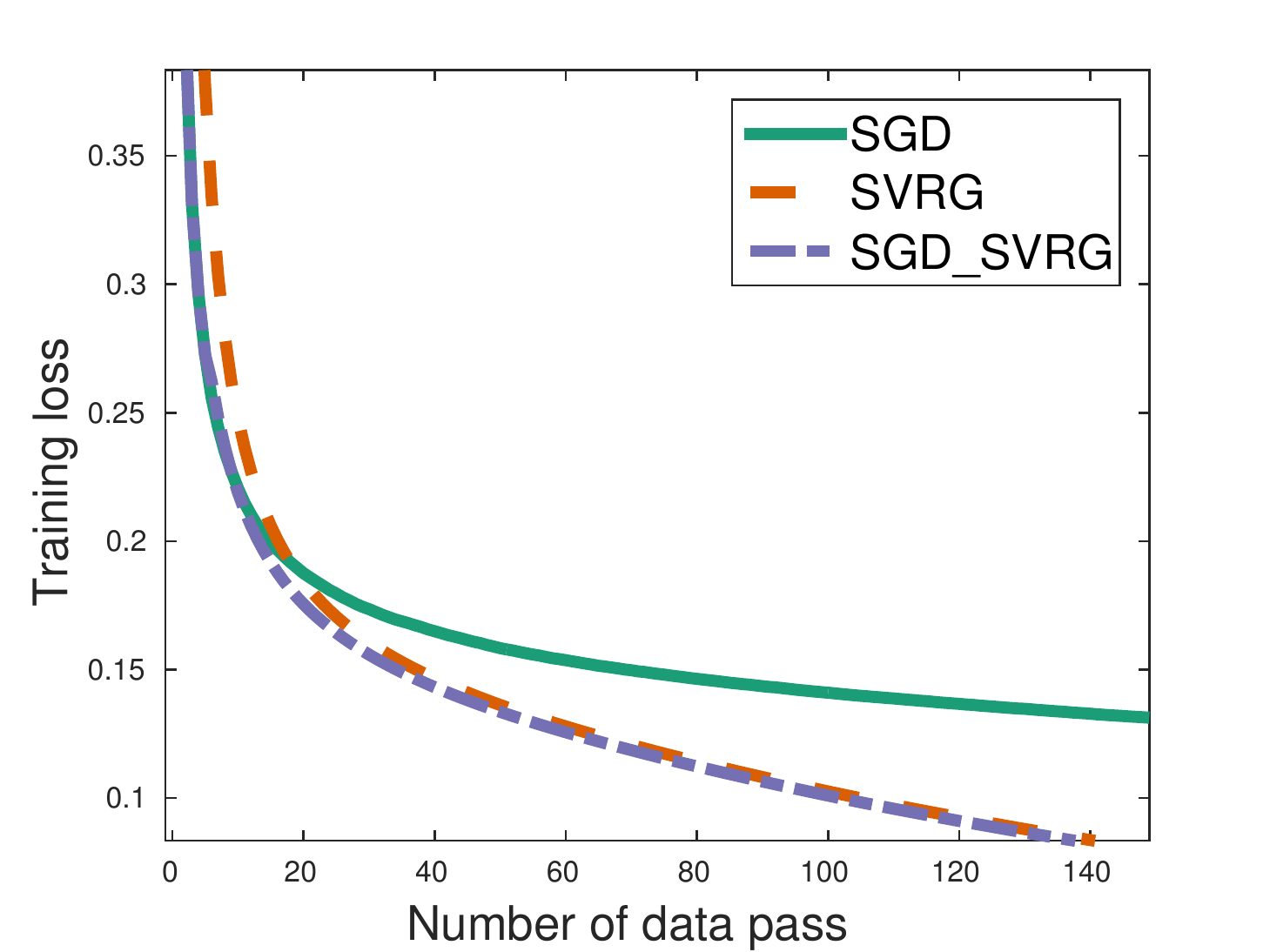}
\caption{}
\label{ompcomp_trainloss}
\end{subfigure}
\begin{subfigure}[b]{0.32\textwidth}
\centering
\includegraphics[width=1.7in]{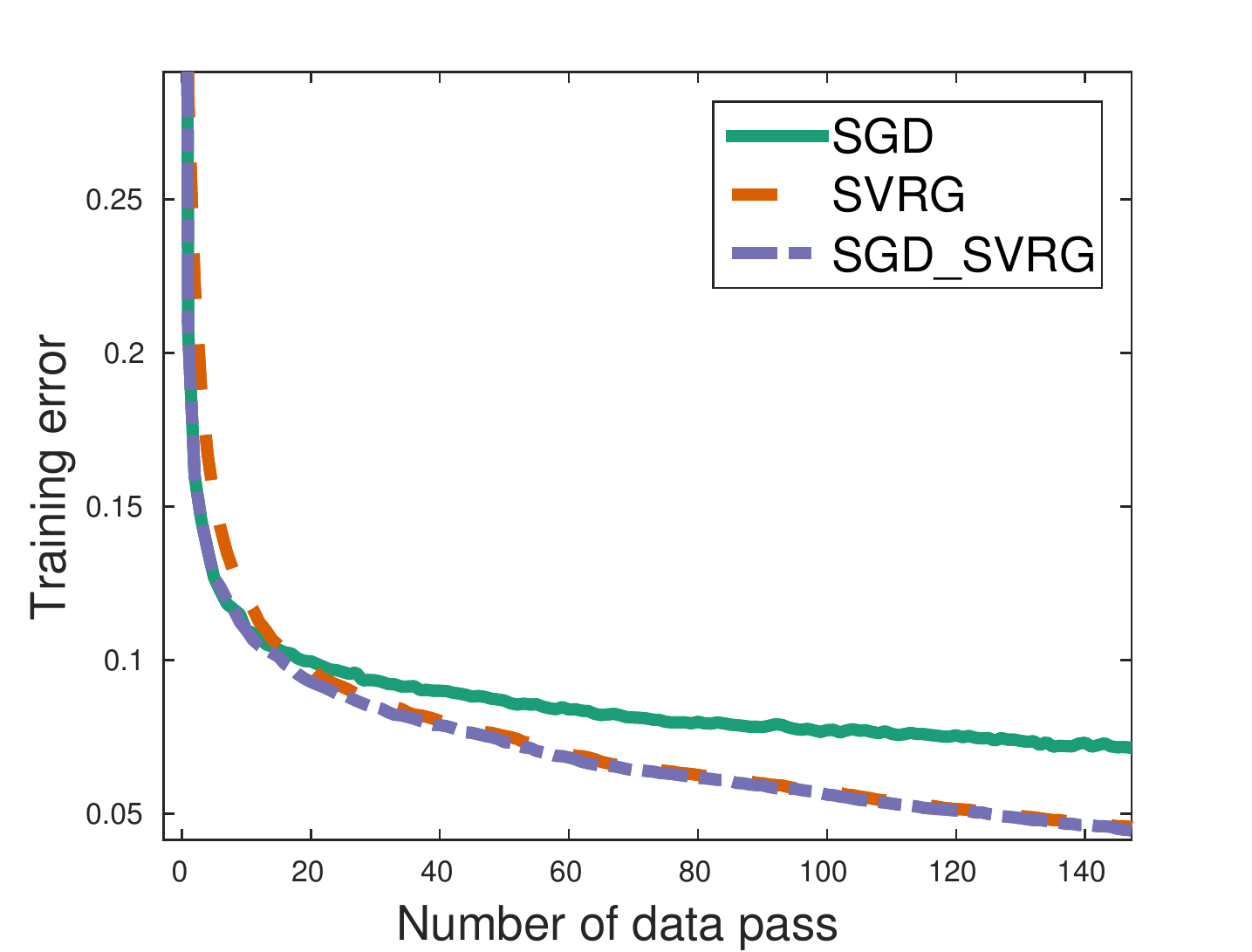}
\caption{}
\label{ompcomp_trainerror}
\end{subfigure}
\begin{subfigure}[b]{0.32\textwidth}
\centering
\includegraphics[width=1.7in]{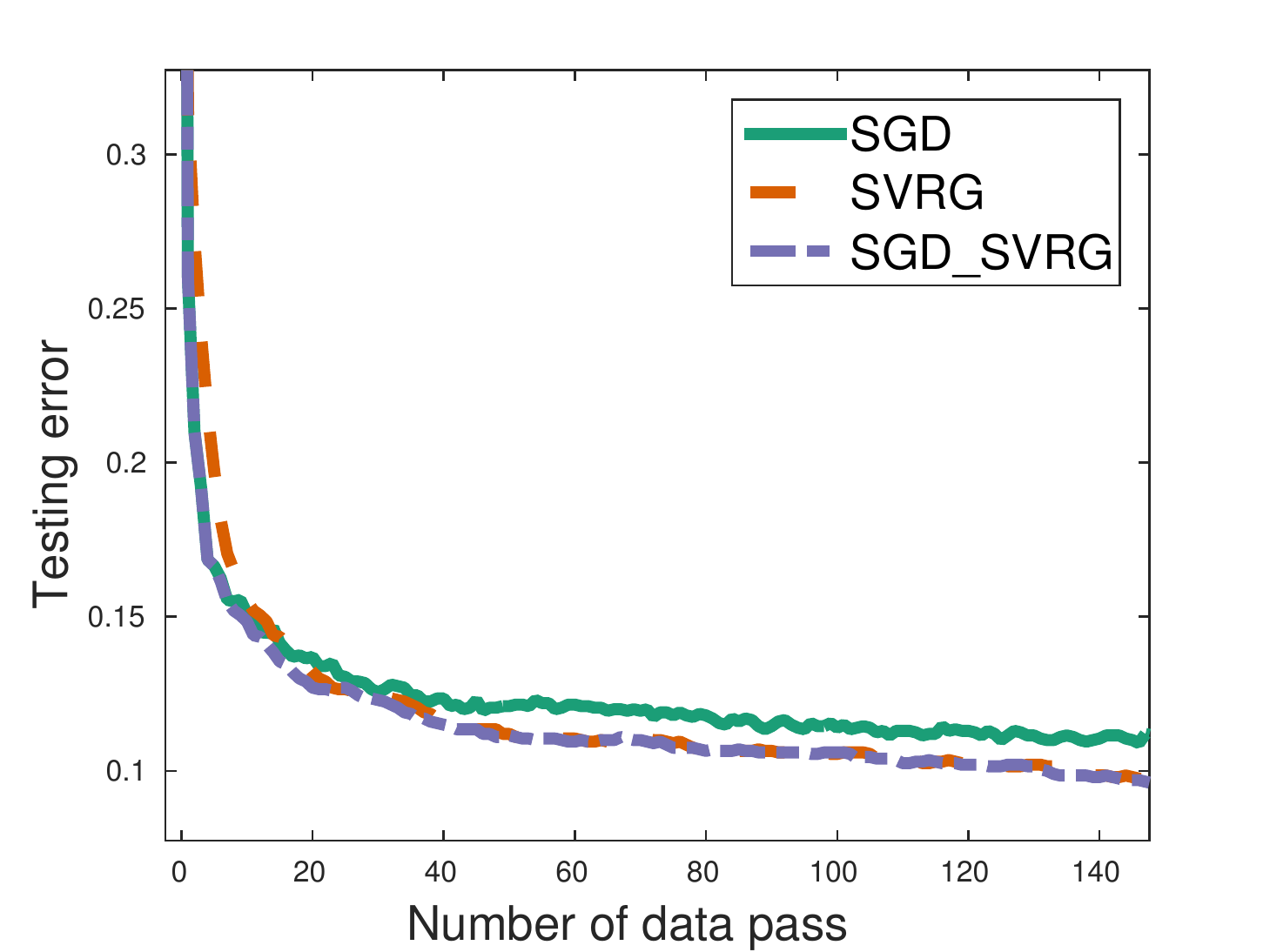}
\caption{}
\label{ompcomp_testerror}
\end{subfigure}
\caption{Comparison of three methods (SGD, SVRG, SGD$\_$SVRG) on MNIST dataset. Figure \ref{ompcomp_trainloss} shows the convergence of loss function value on training dataset.
Figure \ref{ompcomp_trainloss} shows the convergence of training error rate and Figure \ref{ompcomp_testerror} shows the convergence of test error rate. }
\label{ompcomp}
\end{figure}

\begin{figure}[h]
\centering
\begin{subfigure}[b]{0.32\textwidth}
\centering
\includegraphics[width=1.7in]{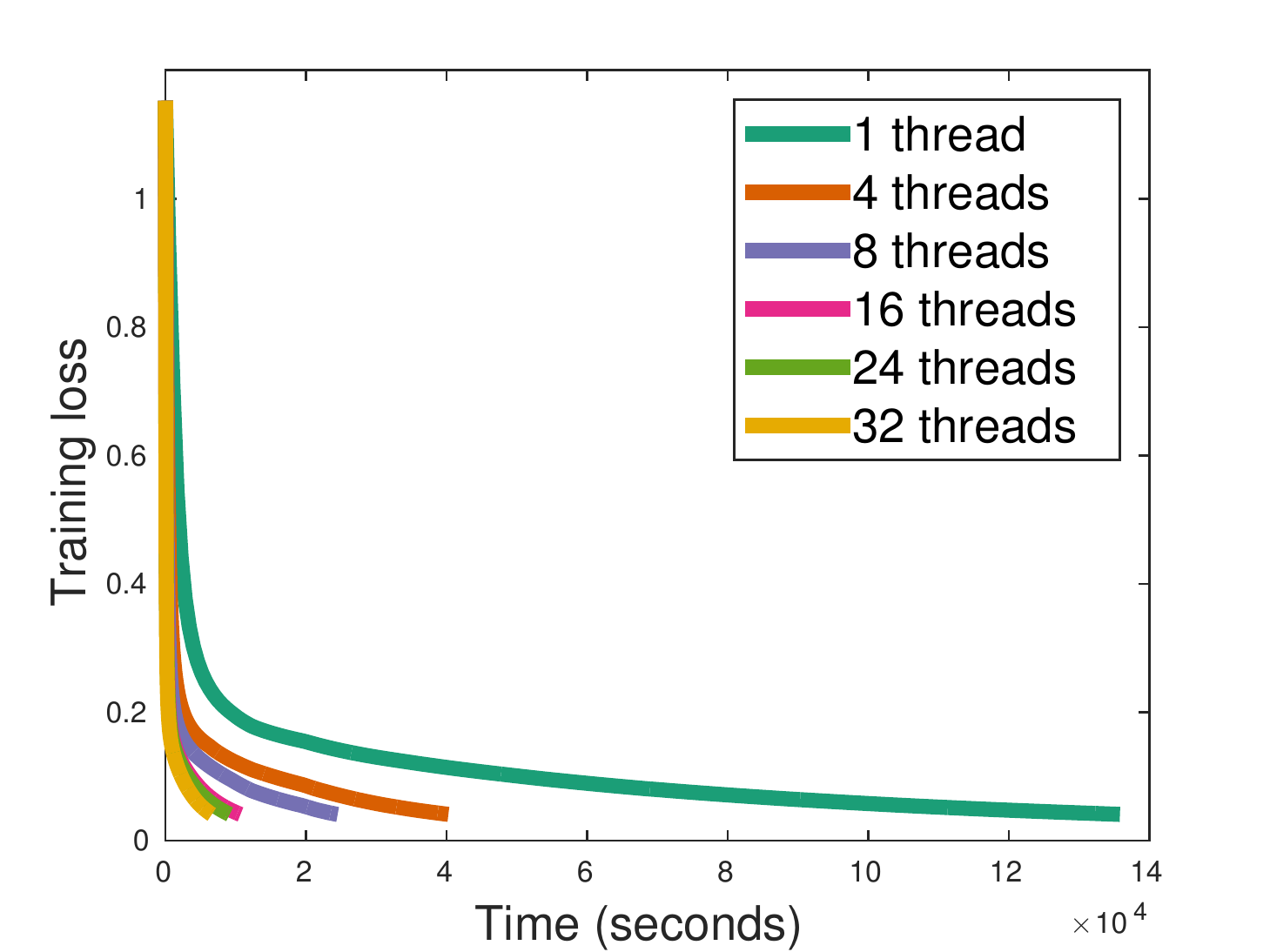}
\caption{}
\label{threads_trainloss}
\end{subfigure}
\begin{subfigure}[b]{0.32\textwidth}
\centering
\includegraphics[width=1.7in]{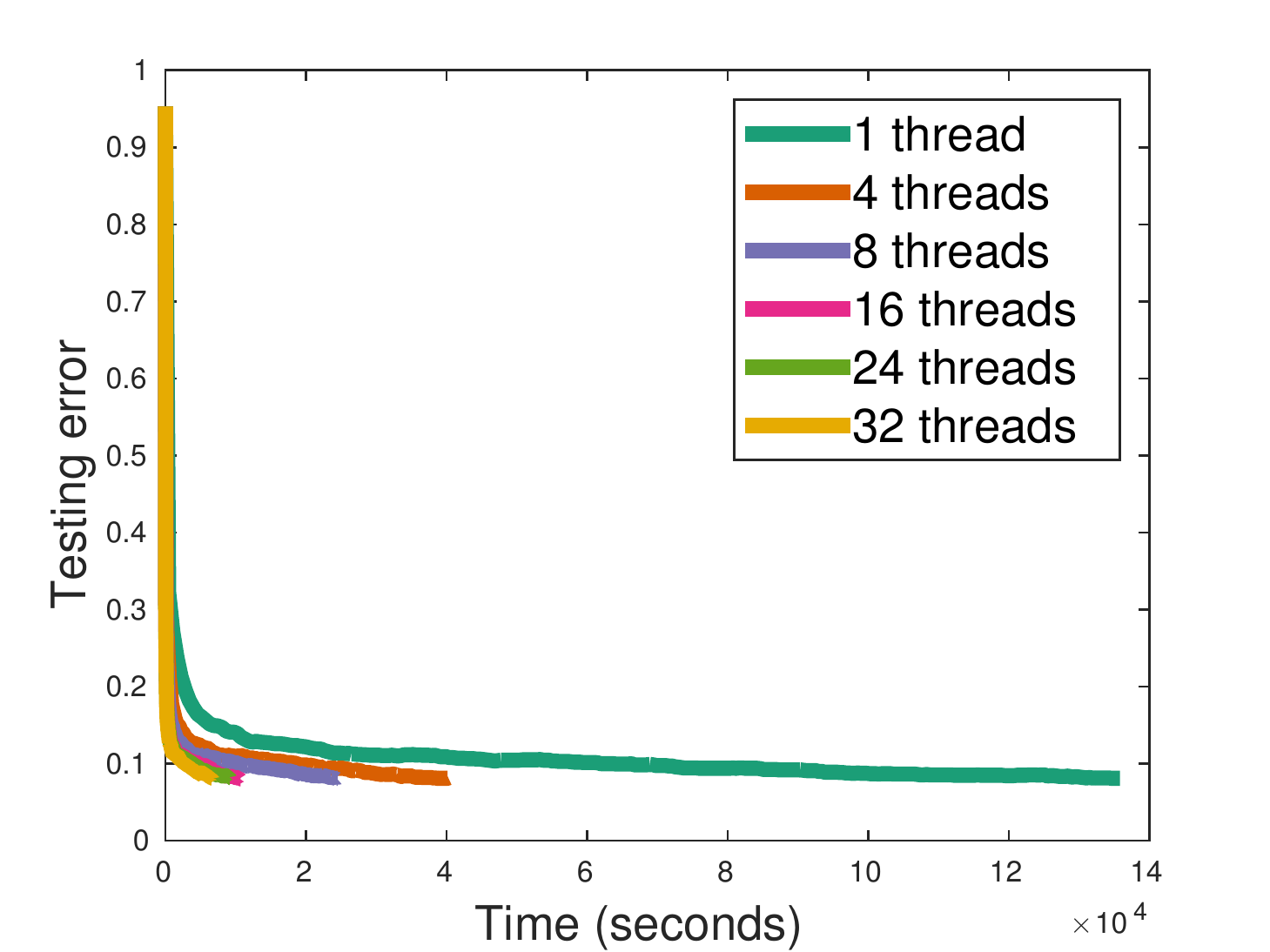}
\caption{}
\label{threads_testerror}
\end{subfigure}
\begin{subfigure}[b]{0.32\textwidth}
\centering
\includegraphics[width=1.7in]{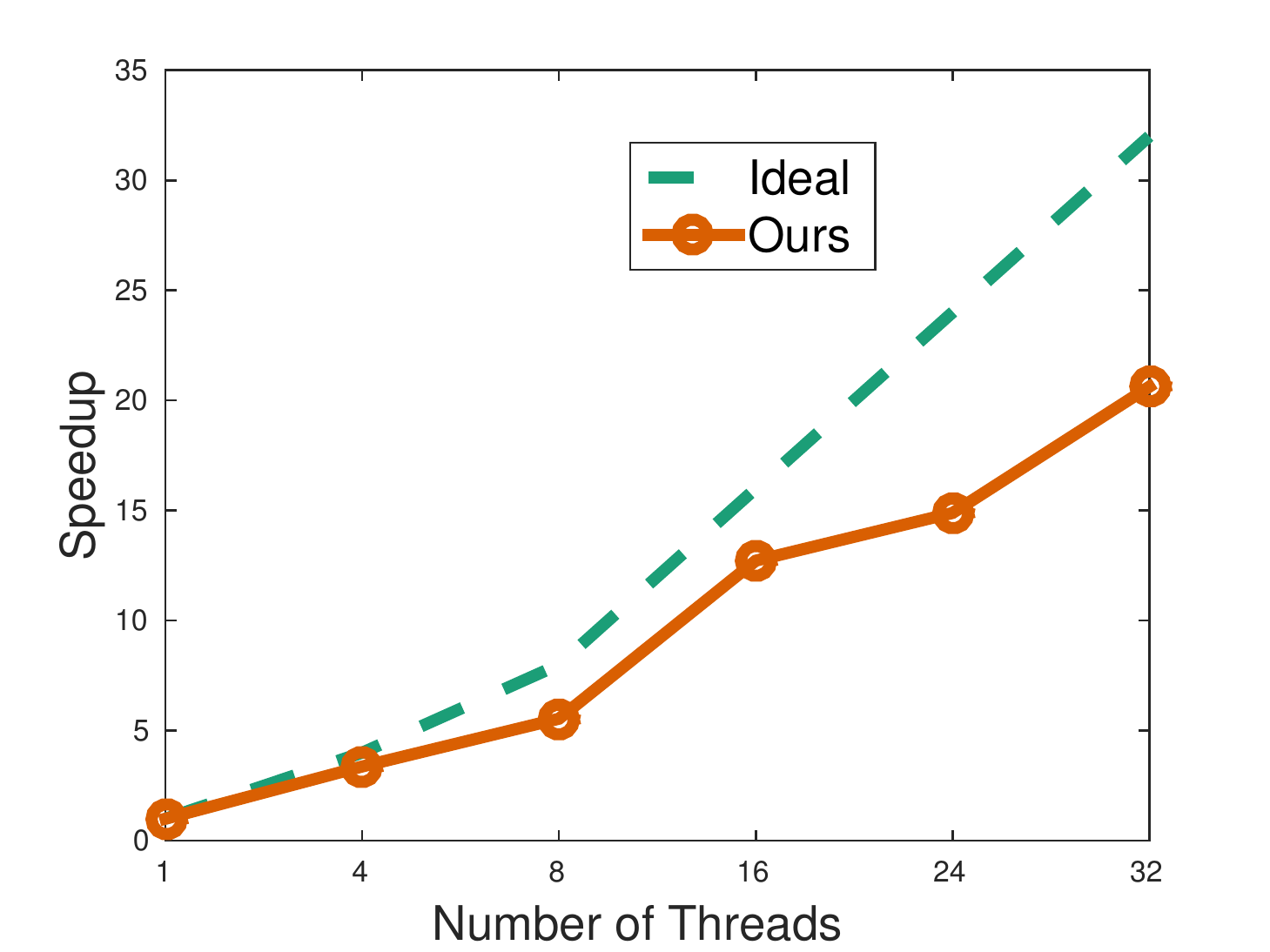}
\caption{}
\label{threads_speedup}
\end{subfigure}
\caption{Asynchronous stochastic gradient descent method with variance reduction runs on a machine using different number of threads from $1$ to $32$. The curves in Figure
\ref{threads_trainloss} shows the convergence of training loss value with respect to time.  The curves in Figure \ref{threads_testerror} shows
the convergence of error rate on testing data. Figure \ref{threads_speedup} represents the running time speedup when we use different workers,
where the dashed line represents ideal linear speedup.
}
\label{threads}
\end{figure}

\subsection{Distributed-memory Architecture}
We conduct distributed-memory architecture experiment on Amazon AWS platform\footnote{https://aws.amazon.com/}, and each node is a t2.micro instance with one CPU. Each server and
worker takes a single node. The point to point communication between server and workers are handled by MPICH library\footnote{http://www.mpich.org/}. CIFAR-10
dataset \cite{krizhevsky2009learning} has $10$ classes of color image $32 \times 32 \times 3$. We use $20000$ samples as training data and $4000$ samples as testing data. We use a
pre-trained CNN model in TensorFlow tutorial \cite{abadi2016tensorflow}, and extract features from second fully connected layer. Thus, each sample is a vector of size $384$. We construct
a three-layer fully connected neural network $(384 \times 50 \times 10)$. We train  this model with softmax loss function, and $\ell_2$ regularization
with weight $C=1e^{-4}$. In this experiment, mini-batch size $|I_t| = 10$, and the inner loop length $m=2,000$. Similar to the compared methods in shared-memory architecture, we implement
SGD method with polynomial learning rate, SVRG with constant learning rate. $\text{SGD}\_\text{SVRG}$ method is initialized with parameters learned after $1$ epoch of SGD method.

At first, we train our model on CIFAR-10 dataset with three compared methods, and each method is with a best tuned learning rate. Performances of all three methods
are presented in Figure \ref{mpicomp}. In Figure
\ref{mpicomp_trainloss}, the curves show that SGD is fast in the first few iterations, and then, SVRG-based method will outperform it due to learning rate issue. As mentioned in
\cite{reddi2016stochastic}, SVRG is more sensitive than SGD to the initial point, so using a pre-trained model is really helpful. It is obvious that SGD$\_$SVRG has better convergence rate than
SVRG method. We can also draw the same conclusion from training error curves with respect to data passes in  Figure \ref{mpicomp_trainerror}. Figure \ref{mpicomp_testerror} represents
that the test error performances of three compared methods are comparable.

\begin{figure}[h]
\centering
\begin{subfigure}[b]{0.32\textwidth}
\centering
\includegraphics[width=1.7in]{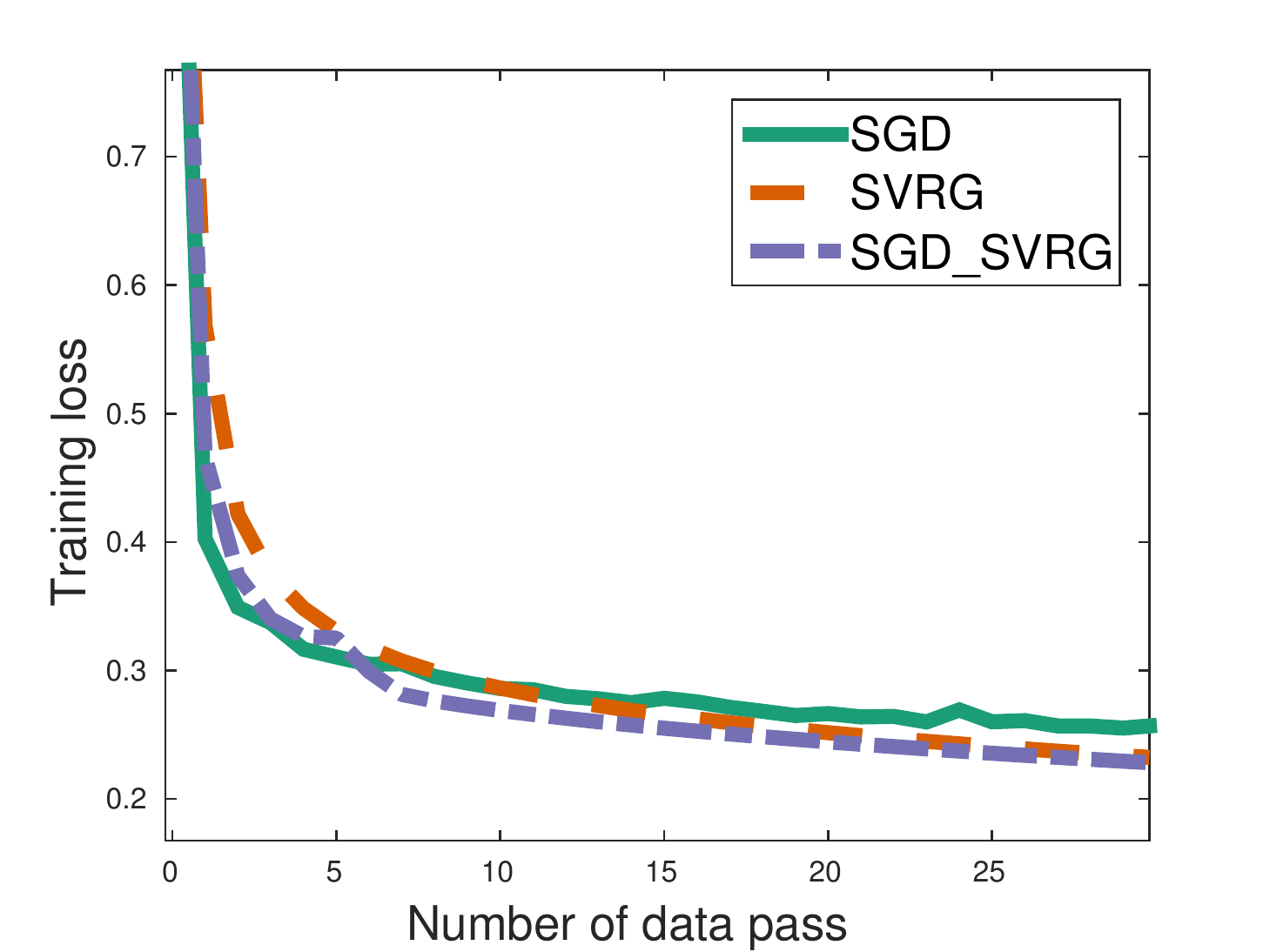}
\caption{}
\label{mpicomp_trainloss}
\end{subfigure}
\begin{subfigure}[b]{0.32\textwidth}
\centering
\includegraphics[width=1.7in]{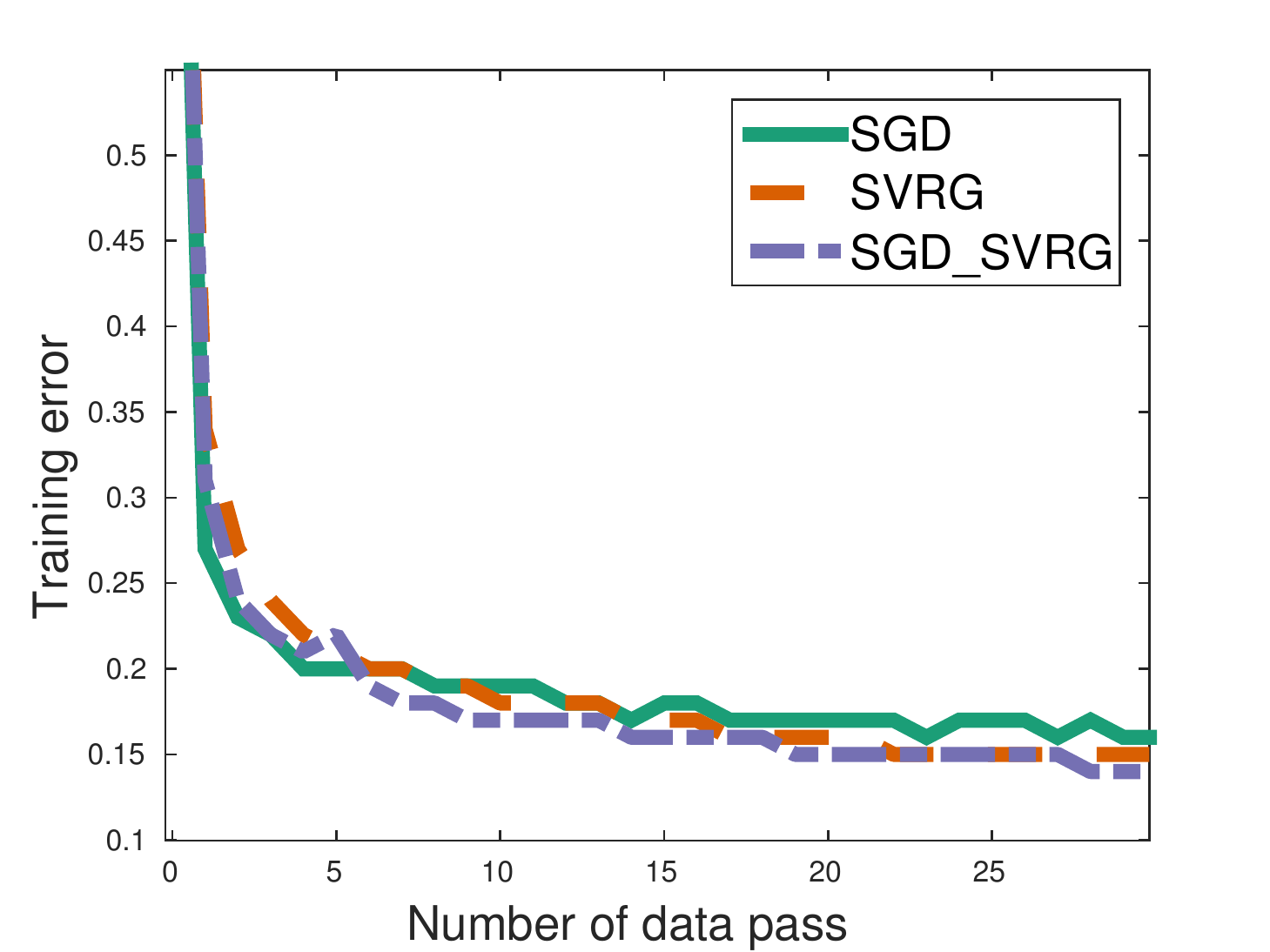}
\caption{}
\label{mpicomp_trainerror}
\end{subfigure}
\begin{subfigure}[b]{0.32\textwidth}
\centering
\includegraphics[width=1.7in]{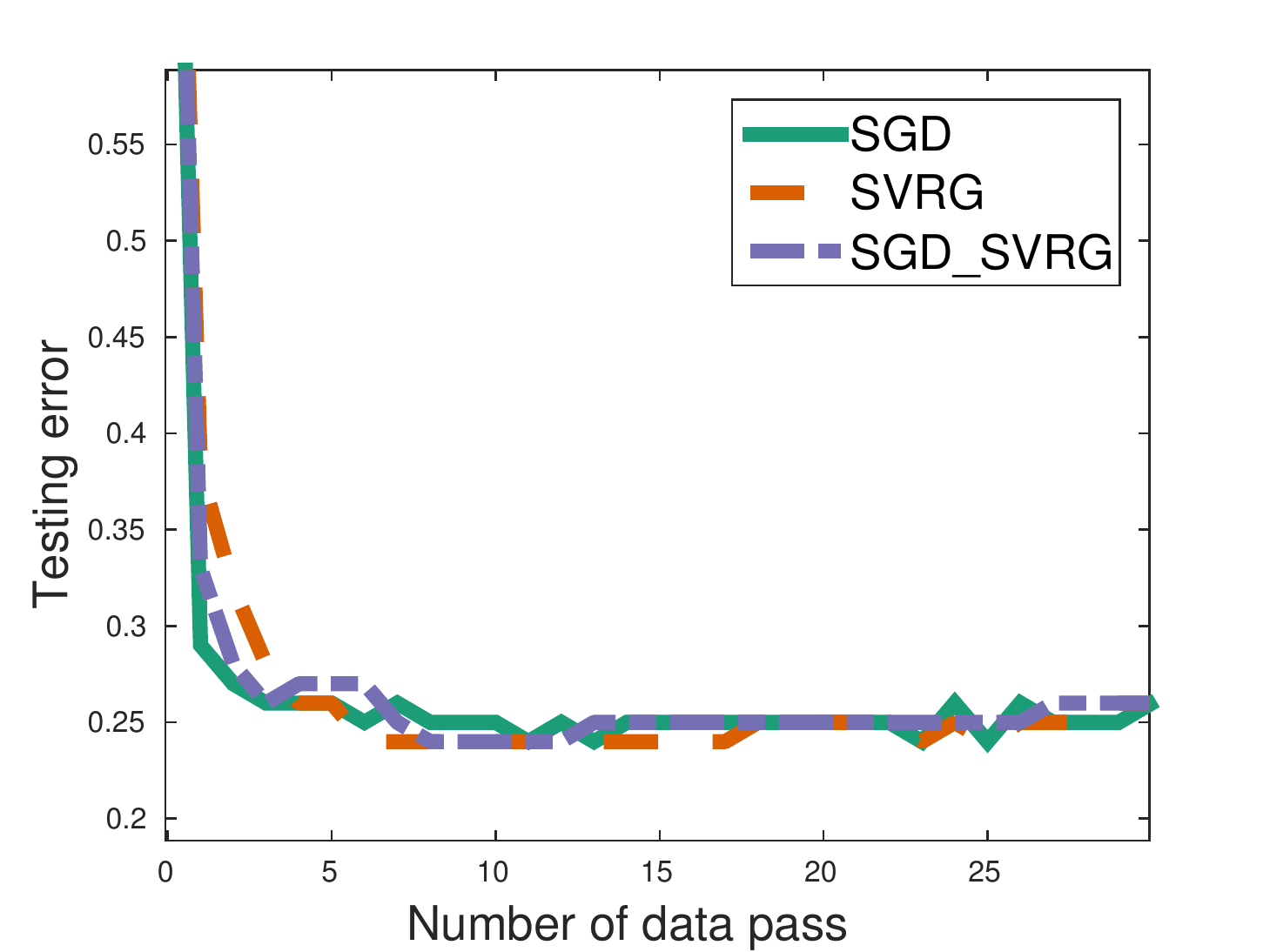}
\caption{}
\label{mpicomp_testerror}
\end{subfigure}
\caption{Comparison of three methods (SGD, SVRG, SGD$\_$SVRG) on CIFAR-10 dataset. Figure \ref{mpicomp_trainloss} shows loss function value on training dataset.
Figure \ref{mpicomp_trainloss} shows the training error and Figure \ref{mpicomp_testerror} shows the test error. }
\label{mpicomp}
\end{figure}

We also test SVRG method with different number of workers, and Figure \ref{mpi} illustrates the results of our experiment. It is easy to draw a conclusion that when the number of
workers increases, we can get a near linear speedup, and when the number gets larger, the speedup tends to be worse.

\begin{figure}[h]
\centering
\begin{subfigure}[b]{0.32\textwidth}
\centering
\includegraphics[width=1.7in]{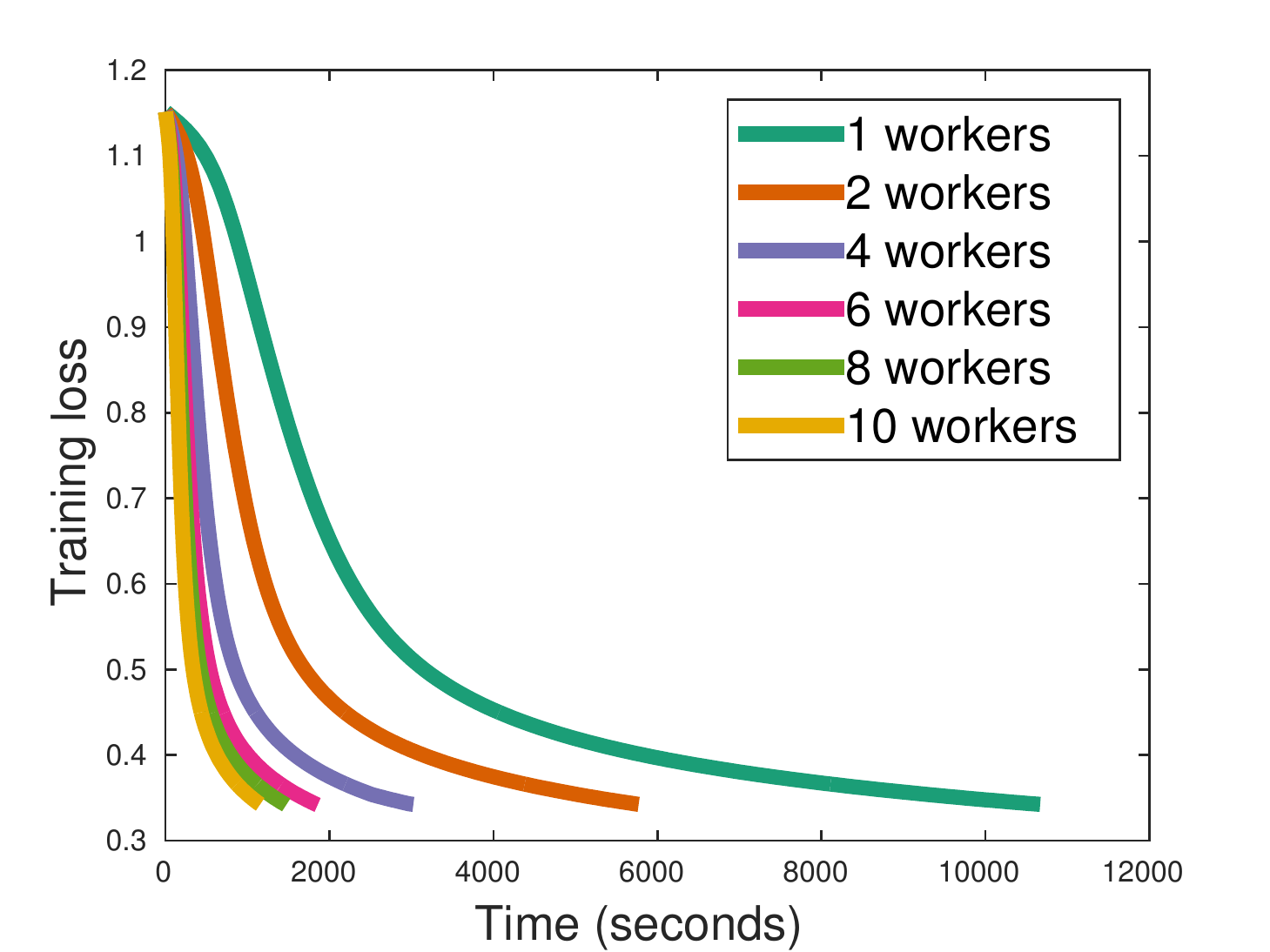}
\caption{}
\label{mpi_trainloss}
\end{subfigure}
\begin{subfigure}[b]{0.32\textwidth}
\centering
\includegraphics[width=1.7in]{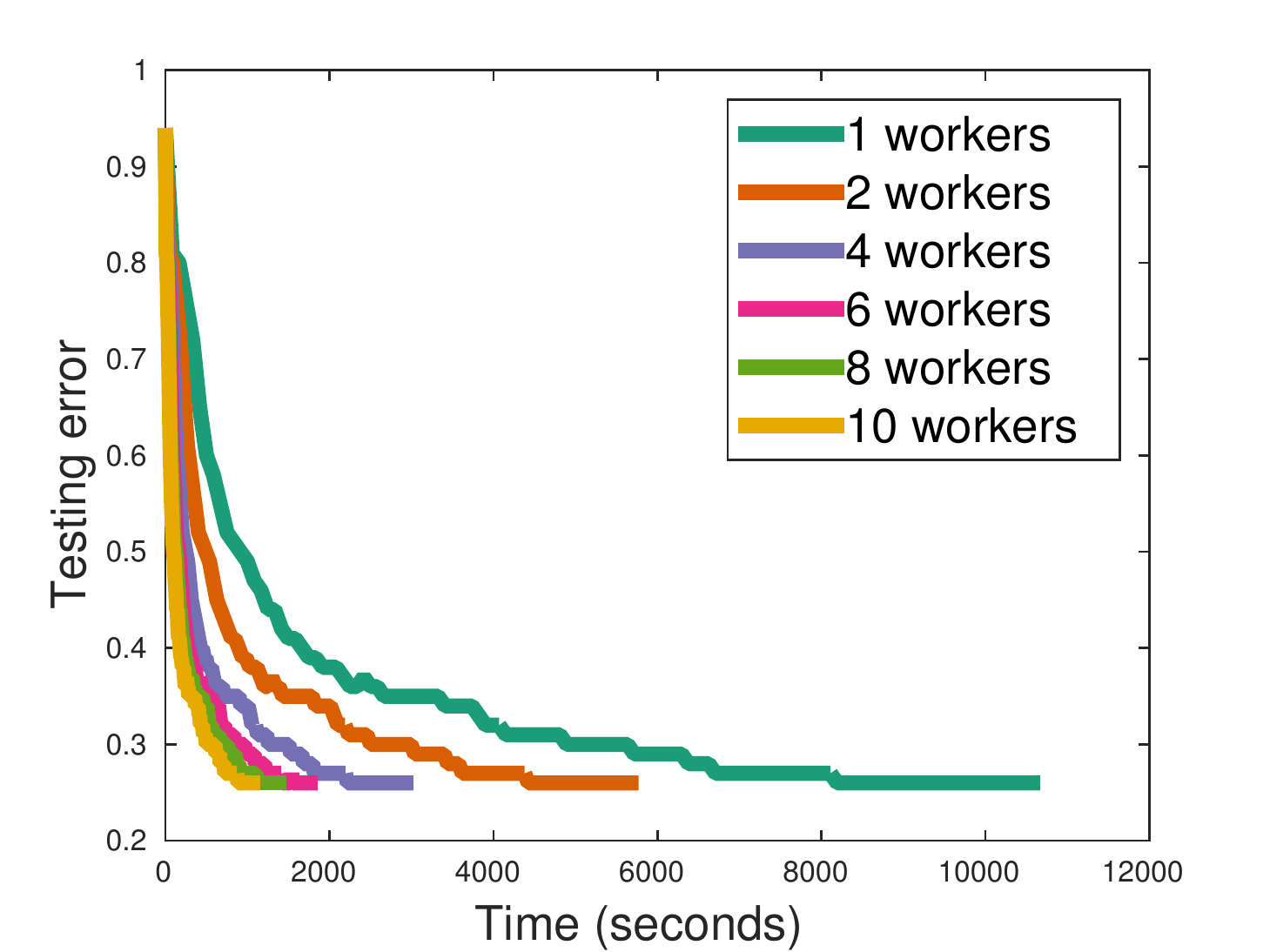}
\caption{}
\label{mpi_testerror}
\end{subfigure}
\begin{subfigure}[b]{0.32\textwidth}
\centering
\includegraphics[width=1.7in]{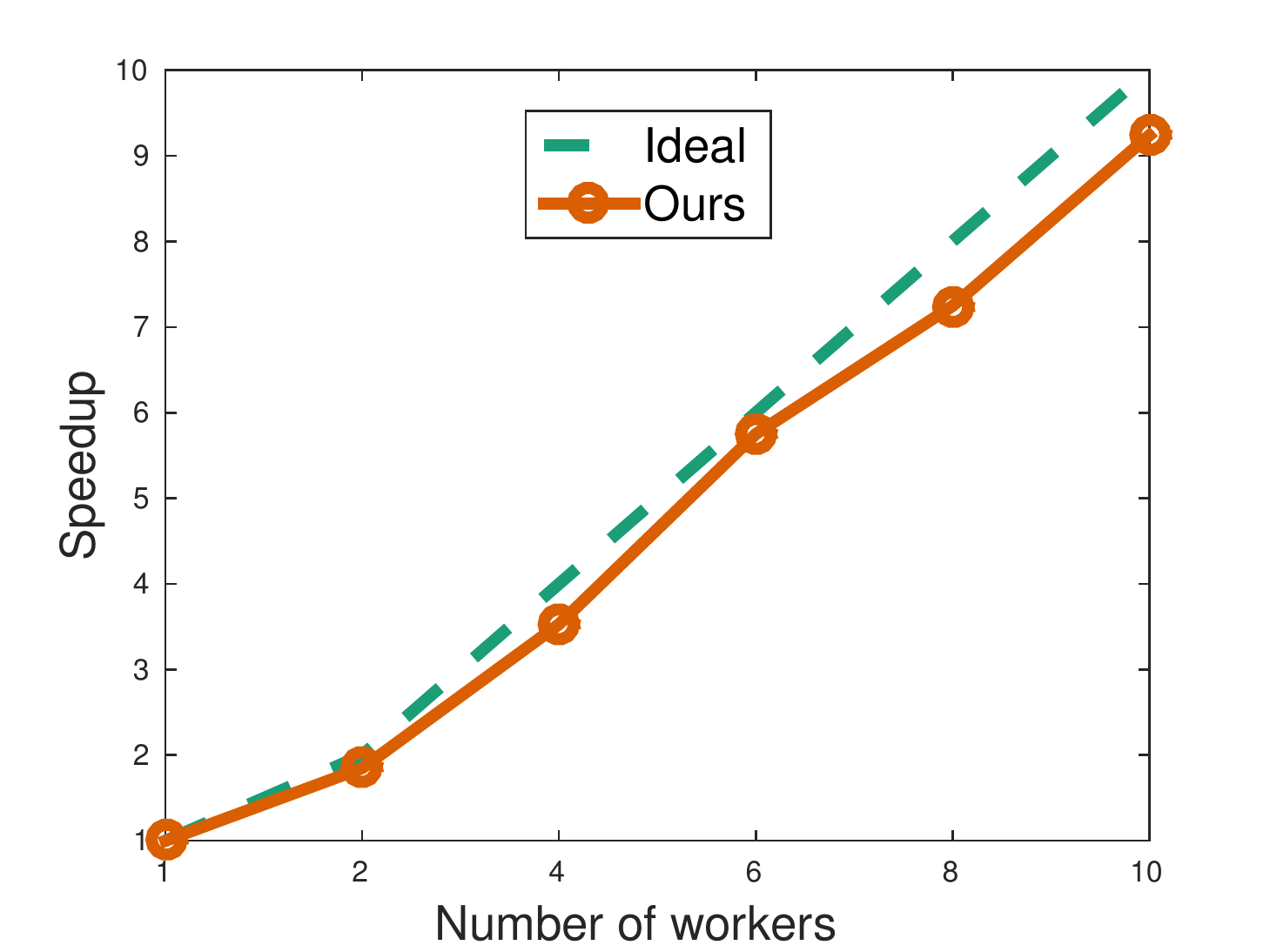}
\caption{}
\label{mpi_speedup}
\end{subfigure}
\caption{Asynchronous stochastic gradient descent method with variance reduction runs on multiple machines from $1$ to $10$. The curves in Figure
\ref{mpi_trainloss} shows the convergence of training loss value with respect to time. The curves in Figure \ref{mpi_testerror} shows the
convergence of error rate on testing data.
Figure \ref{mpi_speedup} represents the running time speedup when using different workers, where the dashed line denotes ideal linear
speedup.
}
\label{mpi}
\end{figure}

\section{Conclusion}
 In this paper, we propose and analyze two different asynchronous stochastic gradient descent with variance reduction for non-convex optimization on two different distributed
 categories, one is shared-memory architecture and the other one is distributed-memory architecture. We analyze their convergence rate and prove that both of them can get an
 ergodic convergence rate $O(1/T)$. Linear speedup is achievable if we increase the number of workers. Experiment results on real dataset also demonstrate our statements.

\bibliographystyle{plain}
\bibliography{ArxivVersion}  

\clearpage
\appendix

\section{Proof of Corollary \ref{lem4}}
\begin{proof}[Proof of Corollary \ref{lem4}]
As per the definitions of $v_t^{s+1}$ (\ref{shared_v_2}) and $u_t^{s+1}$ (\ref{shared_u_2}):
\begin{eqnarray}
\mathbb{E}\left[ ||v_{t}^{s+1}||^2 \right] &=&  \mathbb{E}\left[ || v_{t}^{s+1} - u_t^{s+1} + u_t^{s+1}||^2 \right]  \nonumber \\
&\leq&  2 \mathbb{E}\left[ ||v_{t}^{s+1} - u_t^{s+1} ||^2 \right] +  2 \mathbb{E}\left[ ||u_t^{s+1}||^2 \right] \nonumber \\
&=& 2 \mathbb{E}\left[ || \frac{1}{b}\sum\limits_{i_t \in I_t} \nabla f_{i_t}(\hat x_{t,i_t}^{s+1}) - \nabla f_{i_t}(x_t^{s+1})||^2 \right] +  2\mathbb{E}\left[ ||u_t^{s+1} ||^2 \right] \nonumber \\
&\leq& \frac{ 2L^2}{b} \sum\limits_{i_t \in I_t}  \mathbb{E}\left[ || \hat x_{t,i_t}^{s+1} - x_t^{s+1}||^2 \right] +  2\mathbb{E}\left[ ||u_t^{s+1} ||^2 \right] \nonumber \\
&\leq& \frac{2L^2}{b} \sum\limits_{i_t \in I_t}  \mathbb{E}\left[ || \sum\limits_{j \in J(t,i_t)} (x_{j}^{s+1} - x_{j+1}^{s+1})_{k_j}||^2 \right] +  2\mathbb{E}\left[ ||u_t^{s+1} ||^2 \right] \nonumber \\
&\leq& \frac{2L^2 \Delta \eta^2}{bd} \sum\limits_{i_t \in I_t}  \sum\limits_{j \in J(t,i_t)} \mathbb{E}\left[ ||v_j^{s+1}||^2 \right] +  2\mathbb{E}\left[ ||u_t^{s+1} ||^2 \right] \nonumber \\
\end{eqnarray}
where the first, third and last inequality follows from $||a_1+...+a_n||^2 \leq n\sum\limits_{i=1}^n||a_i||^2 $. Second inequality follows from Lipschitz smoothness of $f(x)$.  Then sum over $\mathbb{E}\left[ ||v_{t}^{s+1}||^2 \right] $ in one epoch, we get the following inequality,

\begin{eqnarray}
\sum\limits_{t=0}^{m-1} \mathbb{E}\left[ ||v_{t}^{s+1}||^2 \right]  &\leq& \sum\limits_{t=0}^{m-1} \left[  \frac{2L^2 \Delta \eta^2}{bd} \sum\limits_{i_t \in I_t}  \sum\limits_{j \in J(t,i_t)} \mathbb{E}\left[ ||v_j^{s+1}||^2 \right]  +  2\mathbb{E}\left[ ||u_t^{s+1} ||^2 \right]  \right] \nonumber \\
&\leq & \frac{ 2L^2 \Delta^2 \eta^2 }{d} \sum\limits_{t=0}^{m-1}   \mathbb{E} \left[ ||v_t^{s+1}||^2 \right] + 2 \sum\limits_{t=0}^{m-1} \mathbb{E}\left[ ||u_t^{s+1} ||^2 \right]
\end{eqnarray}

Thus, if  $d - 2L^2\Delta^2\eta^2 > 0$, then $||v_t^{s+1}||^2$ is upper bounded by $||u_t^{s+1}||^2$,
\begin{eqnarray}
\sum\limits_{t=0}^{m-1} \mathbb{E}\left[ ||v_{t}^{s+1}||^2 \right]  &\leq& \frac{2d}{d-2L^2 \Delta^2 \eta^2}  \sum\limits_{t=0}^{m-1}  \mathbb{E}\left[ ||u_t^{s+1} ||^2 \right]
\end{eqnarray}

We follow the proof in \cite{reddi2015variance}, however, because our update step is different, our result is also a little different.

\end{proof}

\section{Proof of Theorem \ref{thm_m_3}}

\begin{proof}[Proof of Theorem \ref{thm_m_3}]

At first, we derive the upper bound of $\mathbb{E}\left[ ||x_{t+1}^{s+1} - \tilde{x}^s ||^2 \right]$:
\begin{align}
\label{shared_x2_inequality}
\mathbb{E} & \left[ ||x_{t+1}^{s+1} - \tilde{x}^s ||^2 \right]  = \mathbb{E} \left[ ||x_{t+1}^{s+1} - x_t^{s+1} + x_t^{s+1} - \tilde{x}^s ||^2  \right] \nonumber \\
&= \mathbb{E}\left[  ||x_{t+1}^{s+1} -x_t^{s+1} ||^2 + ||x_t^{s+1} - \tilde{x}^s ||^2 + 2\left<x_{t+1}^{s+1} - x_{t}^{s+1}, x_t^{s+1}-\tilde{x}^s\right> \right] \nonumber \\
&= \mathbb{E} \left[ \frac{\eta^2}{d} ||v_t^{s+1} ||^2 + ||x_t^{s+1} - \tilde{x}^s||^2 - \frac{2\eta}{d} \left<  \frac{1}{b} \sum\limits_{i_t \in I_t}  \nabla f(\hat x_{t,i_t}^{s+1} ), x_t^{s+1} - \tilde{x}^s \right>  \right] \nonumber \\
&\leq  \frac{\eta^2}{d}\mathbb{E} \left[ ||v_t^{s+1}||^2 \right] + \frac{2\eta}{d} \mathbb{E} \left[ \frac{1}{2\beta_t}|| \frac{1}{b} \sum\limits_{i_t \in I_t}  \nabla f(\hat x_{t,i_t}^{s+1} )||^2 + \frac{\beta_t}{2}||x_t^{s+1} - \tilde{x}^s||^2 \right] \nonumber \\
& +   \mathbb{E} \left[ ||x_t^{s+1} - \tilde{x}^s||^2 \right] \nonumber \\
&= \frac{\eta^2}{d}\mathbb{E} \left[ ||v_t^{s+1}||^2 \right]   + \frac{\eta}{d \beta_t} \mathbb{E} \left[ || \frac{1}{b} \sum\limits_{i_t \in I_t}  \nabla f(\hat x_{t,i_t}^{s+1} )||^2 \right] + (1+ \frac{\eta \beta_t}{d})  \mathbb{E} \left[ ||x_t^{s+1} - \tilde{x}^s||^2 \right]
\end{align}

where the inequality follows from $ \left<a,b\right> \leq \frac{1}{2}(a^2 + b^2)$. Then we know that $\mathbb{E}\left[ f(x_{t+1}^{s+1}) \right]  $ is also upper bounded:
 \begin{align}
 \label{shared_f}
\mathbb{E}\left[ f(x_{t+1}^{s+1}) \right] & \leq \mathbb{E}\left[ f(x_t^{s+1}) + \left< \nabla f(x_t^{s+1}) , x_{t+1}^{s+1}-x_t^{s+1} \right> + \frac{L}{2} ||  x_{t+1}^{s+1} - x_t^{s+1}||^2 \right] \nonumber \\
&= \mathbb{E}\left[ f(x_t^{s+1}) \right] - \frac{\eta}{d} \mathbb{E}\left[ \left< \nabla f(x_t^{s+1}),  \frac{1}{b} \sum\limits_{i_t \in I_t}  \nabla f(\hat x_{t,i_t}^{s+1} ) \right> \right] + \frac{\eta^2 L}{2d} \mathbb{E}\left[ ||v_{t}^{s+1}||^2\right] \nonumber \\
&= \mathbb{E}\left[ f(x_t^{s+1}) \right] - \frac{\eta}{2d} \mathbb{E}\biggl[  ||\nabla f(x_t^{s+1})||^2 +  ||\frac{1}{b} \sum\limits_{i_t \in I_t}  \nabla f(\hat x_{t,i_t}^{s+1} )||^2 \nonumber \\
&- ||\nabla f(x_t^{s+1}) -\frac{1}{b} \sum\limits_{i_t \in I_t}  \nabla f(\hat x_{t,i_t}^{s+1} ) ||^2   \biggr] + \frac{\eta^2 L}{2d} \mathbb{E}\left[ ||v_{t}^{s+1}||^2\right] \nonumber \\
\end{align}
where the first inequality follows from Lipschitz continuity of $f(x)$.

 \begin{eqnarray}
 \label{shared_f_t1}
\mathbb{E} \left[ ||\nabla f(x_t^{s+1}) - \frac{1}{b} \sum\limits_{i_t \in I_t}  \nabla f(\hat x_{t,i_t}^{s+1} )||^2 \right] &\leq&  \frac{ L^2}{b} \sum\limits_{i_t \in I_t}
\mathbb{E} \left[ ||x_t^{s+1} - \hat x_{t,i_t}^{s+1}||^2 \right] \nonumber \\
 &=& \frac{L^2}{b} \sum\limits_{i_t \in I_t}  \mathbb{E} \left[ ||\sum\limits_{j \in J(t,i_t)} (x_j^{s+1} - x_{j+1}^{s+1}) ||^2 \right] \nonumber \\
 &\leq& \frac{ L^2 \Delta}{b} \sum\limits_{i_t \in I_t}   \sum\limits_{j \in J(t,i_t)} \mathbb{E} \left[ ||x_j^{s+1} - x_{j+1}^{s+1}||^2 \right] \nonumber \\
 &\leq& \frac{L^2  \Delta  \eta^2}{bd} \sum\limits_{i_t \in I_t}  \sum\limits_{j \in J(t,i_t)} \mathbb{E}\left[ ||v_{j}^{s+1}||^2\right]
\end{eqnarray}
where the first inequality follows from Lipschitz continuity of $f(x)$. $\Delta$ denotes the upper bound of  time delay. From (\ref{shared_f}) and (\ref{shared_f_t1}), it is to
derive the following inequality:
\begin{eqnarray}
\label{shared_f_inequality}
\mathbb{E}\left[ f(x_{t+1}^{s+1}) \right]  \leq \mathbb{E}\left[ f(x_t^{s+1}) \right] - \frac{\eta}{2d} \mathbb{E}\left[
||\nabla f(x_t^{s+1})||^2 \right] -  \frac{\eta}{2d} \mathbb{E}\left[  ||\frac{1}{b} \sum\limits_{i_t \in I_t}  \nabla f(\hat x_{t,i_t}^{s+1} )||^2 \right]  \nonumber \\
+ \frac{\eta^2L}{2d} \mathbb{E}\left[ ||v_{t}^{s+1}||^2 \right]  + \frac{L^2 \Delta \eta^3}{2bd^2} \sum\limits_{i_t \in I_t}  \sum\limits_{j \in J(t,i_t)} \mathbb{E}\left[ ||v_{j}^{s+1}||^2\right]
\end{eqnarray}

Following the proof in \cite{reddi2016stochastic}, we define Lyapunov function (this nice proof approach was first introduced in \cite{reddi2016stochastic}):
\begin{eqnarray}
\label{lyapunov}
R_t^{s+1} = \mathbb{E}\left[ f(x_t^{s+1}) + c_t ||x_t^{s+1} - \tilde{x}^s ||^2 \right]\,.
\end{eqnarray}

From the definition of  Lyapunov function, and inequalities in (\ref{shared_x2_inequality}) and (\ref{shared_f_inequality}):
\begin{align}
R_{t+1}^{s+1} &= \mathbb{E}\left[ f(x_{t+1}^{s+1}) + c_{t+1} ||x_{t+1}^{s+1} - \tilde{x}^s ||^2 \right] \nonumber \\
&\leq \mathbb{E}\left[ f(x_t^{s+1}) \right] - \frac{\eta}{2d} \mathbb{E}\left[
||\nabla f(x_t^{s+1})||^2 \right] -  \frac{\eta}{2d} \mathbb{E}\left[  ||\frac{1}{b} \sum\limits_{i_t \in I_t}  \nabla f(\hat x_{t,i_t}^{s+1} )||^2 \right]  \nonumber \\
&+ \frac{\eta^2L}{2d} \mathbb{E}\left[ ||v_{t}^{s+1}||^2 \right]  +  \frac{L^2 \Delta \eta^3}{2bd^2} \sum\limits_{i_t \in I_t}  \sum\limits_{j \in J(t,i_t)} \mathbb{E}\left[ ||v_{j}^{s+1}||^2\right] \nonumber \\
&+ c_{t+1} \left[  \frac{\eta^2}{d}\mathbb{E} \left[ ||v_t^{s+1}||^2 \right]  + (1+ \frac{\eta \beta_t}{d})  \mathbb{E} \left[ ||x_t^{s+1} - \tilde{x}^s||^2 \right]  + \frac{\eta}{d \beta_t} \mathbb{E} \left[ ||\frac{1}{b} \sum\limits_{i_t \in I_t}  \nabla f(\hat x_{t,i_t}^{s+1} )||^2 \right]   \right] \nonumber \\
&= \mathbb{E}\left[ f(x_t^{s+1}) \right]  - \frac{\eta}{2d} \mathbb{E}\left[  ||\nabla f(x_t^{s+1})||^2 \right]  - (\frac{\eta}{2d} - \frac{c_{t+1}\eta}{d\beta_t}) \mathbb{E}\left[  ||\frac{1}{b} \sum\limits_{i_t \in I_t}  \nabla f(\hat x_{t,i_t}^{s+1} )||^2 \right] \nonumber \\
&  +   \frac{L^2 \Delta \eta^3}{2bd^2} \sum\limits_{i_t \in I_t}  \sum\limits_{j \in J(t,i_t)} \mathbb{E}\left[ ||v_{j}^{s+1}||^2\right]   + (\frac{\eta^2L}{2d} + \frac{c_{t+1}\eta^2}{d})  \mathbb{E}\left[ ||v_{t}^{s+1}||^2 \right] \nonumber \\
&+ c_{t+1}(1+ \frac{\eta\beta_t}{d})  \mathbb{E} \left[ ||x_t^{s+1} - \tilde{x}^s||^2 \right]  \nonumber \\
&\leq  \mathbb{E}\left[ f(x_t^{s+1}) \right]  - \frac{\eta}{2d} \mathbb{E}\left[  ||\nabla f(x_t^{s+1})||^2 \right] +  \frac{L^2 \Delta \eta^3}{2bd^2} \sum\limits_{i_t \in I_t}  \sum\limits_{j \in J(t,i_t)} \mathbb{E}\left[ ||v_{j}^{s+1}||^2\right] \nonumber \\
&+ (\frac{\eta^2L}{2d} + \frac{c_{t+1}\eta^2}{d})  \mathbb{E}\left[ ||v_{t}^{s+1}||^2 \right] + c_{t+1}(1+ \frac{\eta\beta_t}{d})  \mathbb{E} \left[ ||x_t^{s+1} - \tilde{x}^s||^2 \right]  \nonumber \\
\end{align}

In the final inequality, we assume $\frac{1}{2} \geq \frac{c_{t+1}}{\beta_t}$. As per Corollary \ref{lem4}, we sum up $R_{t+1}^{s+1}$ from $t=0$ to $m-1$,
\begin{align}
\sum\limits_{t=0}^{m-1} R_{t+1}^{s+1} & \leq  \sum\limits_{t=0}^{m-1} \biggl[ \mathbb{E}\left[ f(x_t^{s+1}) \right]  - \frac{\eta}{2d} \mathbb{E}\left[  ||\nabla f(x_t^{s+1})||^2 \right]
    + \frac{L^2 \Delta \eta^3}{2bd^2} \sum\limits_{i_t \in I_t}  \sum\limits_{j \in J(t,i_t)} \mathbb{E}\left[ ||v_{j}^{s+1}||^2\right] \nonumber \\
&+ (\frac{\eta^2L}{2d} + \frac{c_{t+1}\eta^2}{d})  \mathbb{E}\left[ ||v_{t}^{s+1}||^2 \right] +  c_{t+1}(1+ \frac{\eta\beta_t}{d})  \mathbb{E} \left[ ||x_t^{s+1} - \tilde{x}^s||^2 \right] \biggr] \nonumber \\
& \leq  \sum\limits_{t=0}^{m-1} \biggl[ \mathbb{E}\left[ f(x_t^{s+1}) \right]  - \frac{\eta}{2d} \mathbb{E}\left[  ||\nabla f(x_t^{s+1})||^2 \right]
  +  c_{t+1}(1 + \frac{\eta\beta_t}{d})  \mathbb{E} \left[ ||x_t^{s+1} - \tilde{x}^s||^2 \right] \nonumber \\
&  + (\frac{L^2 \Delta^2 \eta^3}{2d^2} +  \frac{\eta^2L}{2d} + \frac{c_{t+1}\eta^2}{d})  \mathbb{E}\left[ ||v_{t}^{s+1}||^2 \right]    \biggr] \nonumber\\
 & \leq  \sum\limits_{t=0}^{m-1} \biggl[ \mathbb{E}\left[ f(x_t^{s+1}) \right]  - \frac{\eta}{2d} \mathbb{E}\left[  ||\nabla f(x_t^{s+1})||^2 \right]
  +  c_{t+1}(1+ \frac{\eta\beta_t}{d})  \mathbb{E} \left[ ||x_t^{s+1} - \tilde{x}^s||^2 \right]  \nonumber  \\
& +  \frac{2d}{d-2L^2 \Delta^2 \eta^2}(\frac{L^2 \Delta^2 \eta^3}{2d^2} + \frac{\eta^2L}{2d} + \frac{c_{t+1}\eta^2}{d})   \mathbb{E}\left[ ||u_{t}^{s+1}||^2 \right]    \biggr] \nonumber \\
&  \leq \sum\limits_{t=0}^{m-1} R_{t}^{s+1} - \sum\limits_{t=0}^{m-1} \left[  \Gamma _t   \mathbb{E}\left[  ||\nabla f(x_t^{s+1})||^2 \right] \right]
\end{align}
where
\begin{eqnarray}
c_t = c_{t+1}(1+ \frac{\eta\beta_t}{d}) + \frac{4L^2}{(d-2L^2 \Delta^2 \eta^2)b}  (\frac{L^2 \Delta^2 \eta^3}{2d} + \frac{\eta^2L}{2} + c_{t+1}\eta^2)
\end{eqnarray}

\begin{eqnarray}
 \Gamma _t =  \frac{\eta}{2d} -\frac{4}{d-2L^2 \Delta^2 \eta^2}  (\frac{L^2 \Delta^2 \eta^3}{2d} + \frac{\eta^2L}{2} + c_{t+1}\eta^2)
\end{eqnarray}

Setting $c_{m} = 0$, $\tilde{x}^{s+1} = x^{s+1}_{m}$, and $\gamma = \min \Gamma_t$,  then
$R_{m}^{s+1} =  \mathbb{E}\left[  f(x_m^{s+1}) \right] =  \mathbb{E}\left[  f(\tilde x^{s+1}) \right]$
and $R_{0}^{s+1} =  \mathbb{E}\left[  f(x_0^{s+1}) \right] =  \mathbb{E}\left[  f(\tilde x^{s}) \right]$. Thus we can get,

\begin{eqnarray}
\sum\limits_{t=0}^{m-1} \mathbb{E}\left[  ||\nabla f(x_t^{s+1})||^2 \right]  \leq \frac{\mathbb{E}\left[  f(\tilde x^{s})  -  f(\tilde x^{s+1}) \right] }{\gamma}
\end{eqnarray}

Summing up all epochs, and define $x^0$ as initial point and $x^*$ as optimal solution, we have the final inequality:
\begin{eqnarray}
\frac{1}{T}\sum\limits_{s=0}^{S-1}\sum\limits_{t=0}^{m-1} \mathbb{E}\left[  ||\nabla f(x_t^{s+1})||^2 \right]  \leq \frac{\mathbb{E}\left[  f( x^{0})  -  f( x^{*}) \right] }{T\gamma}
\end{eqnarray}

\end{proof}

\section{Proof of Theorem \ref{thm_m_4}}

\begin{proof}[Proof of Theorem \ref{thm_m_4}]
Following the proof in \cite{reddi2016stochastic}, we set $c_m=0$, $ \eta = \frac{u_0b}{Ln^\alpha}$, $\beta_t = \beta = {2L}$,   $0<u_0<1$, and $0<\alpha<1$.
\begin{eqnarray}
\theta &= &\frac{\eta \beta}{d} + \frac{4L^2\eta^2}{(d-2L^2\Delta^2\eta^2)b} \nonumber \\
&=&  \frac{2u_0b}{dn^{{\alpha}}} + \frac{4u_0^2b}{dn^{2\alpha} - 2\Delta^2u_0^2b^2} \nonumber \\
&\leq&    \frac{6u_0b} {dn^{{\alpha}}}
\end{eqnarray}

In the  final inequality, we constrain that  $dn^{{\alpha}} \leq dn^{2\alpha} - 2\Delta^2u_0^2b^2 $, and it is easy to satisfy when $n$ is large. We set $m = \lfloor  \frac{dn^{{\alpha}}}{6u_0b}  \rfloor$, and from the recurrence formula of $c_t$, we have:
\begin{eqnarray}
c_0& =& \frac{2L^2}{(d-2L^2\Delta^2\eta^2)b} \left( \frac{ L^2\Delta^2\eta^3}{d} + \eta^2L \right) \frac{(1+\theta)^m - 1}{\theta} \nonumber \\
&=& \frac{2L \left( \frac{u_0^3\Delta^2b^3}{n^{3\alpha}} + \frac{u_0^2b^2d}{n^{2\alpha}}  \right)}{ \left(d - 2 L^2 \Delta^2\eta^2  \right)  \left(\frac{2u_0b^2}{dn^{\alpha}} + \frac{4u_0^2b^2}{dn^{2\alpha} - 2\Delta^2u_0^2b^2}   \right)d} \left( (1+\theta)^m -1 \right) \nonumber \\
&\leq& \frac{L(u_0b\Delta^2 +d)}{3d}  \left( (1+\theta)^m -1 \right) \nonumber \\
&\leq & \frac{L(u_0b\Delta^2 +d)}{3d}(e-1) \nonumber \\
\end{eqnarray}
where the final inequality follows from that $(1+\frac{1}{l})^l$ is increasing for $l>0$, and $\lim\limits_{l\rightarrow  \infty}(1 + \frac{1}{l})^l = e$. From the
proof in Theorem \ref{thm_m_3}, we know that $c_0 \leq \frac{\beta}{2} = {L}$,thus $\Delta^2 \leq \frac{d}{2u_0b}$.
 $c_t$ is decreasing with respect to $t$, and $c_0$ is also upper bounded.

\begin{eqnarray}
\gamma&=& \min_t \Gamma_t \nonumber \\
&\geq &  \frac{\eta}{2d} -\frac{4}{d-2L^2 \Delta^2 \eta^2}  (\frac{L^2 \Delta^2 \eta^3}{2d} + \frac{\eta^2L}{2} + c_{0}\eta^2)   \nonumber \\
&\geq &  \frac{\eta}{2d} -\frac{4n^{\alpha}}{d}  (\frac{L^2 \Delta^2 \eta^3}{2d} + \frac{\eta^2L}{2} + c_{0}\eta^2)   \nonumber \\
&\geq&  \left( \frac{1}{2} - \frac{14u_0^2b^2\Delta^2 + 14u_0bd}{3d} \right)\frac{\eta}{d} \nonumber \\
&\geq& \frac{\sigma b}{dLn^{\alpha}}
\end{eqnarray}

There exists a small value $\sigma$, an it is independent of $n$. The final inequality holds if
$\frac{1}{2} > \frac{14u_0^2b^2\Delta^2 + 14u_0bd}{3d}  $.
Above all, if $\Delta^2 < \min \{\frac{d}{2u_0b}, \frac{3d-28u_0bd}{28u_0^2b^2}  \}$, we have the conclusion that,
\begin{eqnarray}
\frac{1}{T}\sum\limits_{s=0}^{S-1}\sum\limits_{t=0}^{m-1} \mathbb{E}\left[  ||\nabla f(x_t^{s+1})||^2 \right]  \leq \frac{dLn^{\alpha}\mathbb{E}\left[  f(\tilde x^{0})  -  f(\tilde x^{*}) \right] }{T \sigma b}
\end{eqnarray}

\end{proof}

\section{Proof of Theorem \ref{thm_m_1}}

\begin{proof}[Proof of Theorem \ref{thm_m_1}]
\begin{align}
&\mathbb{E}\left[ ||x_{t+1}^{s+1} - \tilde{x}^s ||^2 \right]  = \mathbb{E} \left[ ||x_{t+1}^{s+1} - x_t^{s+1} + x_t^{s+1} - \tilde{x}^s ||^2  \right] \nonumber \\
&= \mathbb{E}\left[  ||x_{t+1}^{s+1} -x_t^{s+1} ||^2 + ||x_t^{s+1} - \tilde{x}^s ||^2 + 2\left<x_{t+1}^{s+1} - x_{t}^{s+1}, x_t^{s+1}-\tilde{x}^s\right> \right] \nonumber \\
&= \mathbb{E} \left[ \eta^2 ||v_t^{s+1} ||^2 + ||x_t^{s+1} - \tilde{x}^s||^2 - 2\eta \left< \frac{1}{b}\sum\limits_{i_t \in I_t} \nabla f(x_{t-\tau_i}^{s+1}), x_t^{s+1} - \tilde{x}^s \right>  \right] \nonumber \\
&\leq \eta^2 \mathbb{E} \left[ ||v_t^{s+1}||^2 \right] + 2 \eta \mathbb{E} \left[ \frac{1}{2\beta_t}|| \frac{1}{b}\sum\limits_{i_t \in I_t} \nabla f(x_{t-\tau_i}^{s+1})||^2 + \frac{\beta_t}{2}||x_t^{s+1} - \tilde{x}^s||^2 \right] \nonumber \\
& + \mathbb{E} \left[ ||x_t^{s+1} - \tilde{x}^s||^2 \right]  \nonumber \\
&= \eta^2 \mathbb{E} \left[ ||v_t^{s+1}||^2 \right]  + (1+\eta \beta_t)  \mathbb{E} \left[ ||x_t^{s+1} - \tilde{x}^s||^2 \right]  + \frac{\eta}{\beta_t} \mathbb{E} \left[ ||\frac{1}{b}\sum\limits_{i_t \in I_t} \nabla f(x_{t-\tau_i}^{s+1})||^2 \right]
\end{align}
where the first inequality follows $2\left<a,b\right> \leq ||a||^2 + ||b||^2 $

 \begin{align}
&\mathbb{E}\left[ f(x_{t+1}^{s+1}) \right]  \leq \mathbb{E}\left[ f(x_t^{s+1}) + \left<\nabla f(x_t^{s+1}) , x_{t+1}^{s+1}-x_t^{s+1} \right> + \frac{L}{2} ||x_{t+1}^{s+1} - x_t^{s+1}||^2 \right] \nonumber \\
&= \mathbb{E}\left[ f(x_t^{s+1}) \right] - \eta \mathbb{E}\left[ \left< \nabla f(x_t^{s+1}),\frac{1}{b}\sum\limits_{i_t \in I_t} \nabla f(x_{t-\tau_i}^{s+1}) \right> \right] + \frac{\eta^2 L}{2} \mathbb{E}\left[ ||v_{t}^{s+1}||^2\right] \nonumber \\
&= - \frac{\eta}{2} \mathbb{E}\left[  ||\nabla f(x_t^{s+1})||^2 +  ||\frac{1}{b}\sum\limits_{i_t \in I_t} \nabla f(x_{t-\tau_i}^{s+1})||^2 -  ||\nabla f(x_t^{s+1}) - \frac{1}{b}\sum\limits_{i_t \in I_t} \nabla f(x_{t-\tau_i}^{s+1})||^2   \right]  \nonumber  \\
& + \mathbb{E}\left[ f(x_t^{s+1}) \right] +  \frac{\eta^2 L}{2} \mathbb{E}\left[ ||v_{t}^{s+1}||^2\right]
\end{align}
where the first inequality follows from Lipschitz continuity of $f(x)$.

 \begin{eqnarray}
 ||\nabla f(x_t^{s+1}) - \frac{1}{b}\sum\limits_{i_t \in I_t} \nabla f(x_{t-\tau_i}^{s+1})||^2 &\leq&  \frac{1}{b}\sum\limits_{i_t \in I_t} || \nabla f(x_t^{s+1}) - \nabla f(x_{t-\tau_i}^{s+1})||^2 \nonumber \\
& \leq& \frac{L^2}{b}\sum\limits_{i_t \in I_t} ||  x_t^{s+1} -  x_{t-\tau_i}^{s+1}||^2 \nonumber \\
& = &  \frac{L^2}{b}\sum\limits_{i_t \in I_t} ||  \sum\limits_{j=t-\tau_i}^{t-1} (x_{j}^{s+1} -  x_{j+1}^{s+1})||^2 \nonumber \\
&\leq & \frac{L^2 \Delta}{b}\sum\limits_{i_t \in I_t}  \sum\limits_{j=t-\tau_i}^{t-1}  || x_{j}^{s+1} -  x_{j+1}^{s+1}||^2 \nonumber \\
& = &  \frac{L^2 \Delta \eta^2}{b}\sum\limits_{i_t \in I_t}  \sum\limits_{j=t-\tau_i}^{t-1}  ||v_j^{s+1}||^2
 \end{eqnarray}
where the second inequality follows from  Lipschitz continuity of $f(x)$.
$\Delta$ denotes the upper bound of  time delay. $\tau \leq \Delta$. Above all, we have the following inequality,
\begin{eqnarray}
\mathbb{E}\left[ f(x_{t+1}^{s+1}) \right]  \leq \mathbb{E}\left[ f(x_t^{s+1}) \right] - \frac{\eta}{2} \mathbb{E}\left[
||\nabla f(x_t^{s+1})||^2 \right] -  \frac{\eta}{2} \mathbb{E}\left[  ||\frac{1}{b}\sum\limits_{i \in I_t} \nabla f(x_{t-\tau_i}^{s+1})||^2 \right]  \nonumber \\
+ \frac{\eta^2L}{2} \mathbb{E}\left[ ||v_{t}^{s+1}||^2 \right]  + \frac{L^2 \Delta \eta^3}{2b}\sum\limits_{i \in I_t}  \sum\limits_{j=t-\tau_i}^{t-1} \mathbb{E}\left[ ||v_j^{s+1}||^2 \right]
\end{eqnarray}
Following the definition of $R_{t+1}^{s+1}$ in \cite{reddi2016stochastic}, 
\begin{align}
&R_{t+1}^{s+1} = \mathbb{E}\left[ f(x_{t+1}^{s+1}) + c_{t+1} ||x_{t+1}^{s+1} - \tilde{x}^s ||^2 \right] \nonumber \\
&\leq  \mathbb{E}\left[ f(x_t^{s+1}) \right] - \frac{\eta}{2} \mathbb{E}\left[  ||\nabla f(x_t^{s+1})||^2 \right] -  \frac{\eta}{2} \mathbb{E}\left[  ||\frac{1}{b}\sum\limits_{i \in I_t} \nabla f(x_{t-\tau_i}^{s+1})||^2 \right] \nonumber \\
&+  \frac{\eta^2L}{2} \mathbb{E}\left[ ||v_{t}^{s+1}||^2 \right]  + \frac{L^2 \Delta \eta^3}{2b}\sum\limits_{i \in I_t}  \sum\limits_{j=t-\tau_i}^{t-1} \mathbb{E}\left[ ||v_j^{s+1}||^2 \right] \nonumber \\
&+ c_{t+1} \left[  \eta^2 \mathbb{E} \left[ ||v_t^{s+1}||^2 \right]  + (1+\eta \beta_t)  \mathbb{E} \left[ ||x_t^{s+1} - \tilde{x}^s||^2 \right]  + \frac{\eta}{\beta_t} \mathbb{E} \left[ ||\frac{1}{b}\sum\limits_{i \in I_t} \nabla f(x_{t-\tau_i}^{s+1})||^2 \right]  \right] \nonumber \\
&=  \mathbb{E}\left[ f(x_t^{s+1}) \right]  - \frac{\eta}{2} \mathbb{E}\left[  ||\nabla f(x_t^{s+1})||^2 \right]  - (\frac{\eta}{2} - \frac{c_{t+1}\eta}{\beta_t}) \mathbb{E}\left[  ||\frac{1}{b}\sum\limits_{i \in I_t} \nabla f(x_{t-\tau_i}^{s+1})||^2 \right] \nonumber \\
&  +  \frac{L^2 \Delta \eta^3}{2b}\sum\limits_{i \in I_t}  \sum\limits_{j=t-\tau_i}^{t-1} \mathbb{E}\left[ ||v_j^{s+1}||^2 \right]   + (\frac{\eta^2L}{2} + c_{t+1}\eta^2)  \mathbb{E}\left[ ||v_{t}^{s+1}||^2 \right]    \nonumber \\
& + c_{t+1}(1+\eta\beta_t)  \mathbb{E} \left[ ||x_t^{s+1} - \tilde{x}^s||^2 \right] \nonumber \\
&\leq  \mathbb{E}\left[ f(x_t^{s+1}) \right]  - \frac{\eta}{2} \mathbb{E}\left[  ||\nabla f(x_t^{s+1})||^2 \right] +   \frac{L^2 \Delta \eta^3}{2b}\sum\limits_{i \in I_t}  \sum\limits_{j=t-\tau_i}^{t-1} \mathbb{E}\left[ ||v_j^{s+1}||^2 \right]  \nonumber \\
&   + (\frac{\eta^2L}{2} + c_{t+1}\eta^2)  \mathbb{E}\left[ ||v_{t}^{s+1}||^2 \right]   +  c_{t+1}(1+\eta\beta_t)  \mathbb{E} \left[ ||x_t^{s+1} - \tilde{x}^s||^2 \right]
\end{align}

In the final inequality, we make $ (\frac{\eta}{2} - \frac{c_{t+1}\eta}{\beta_t}) > 0$. Then we sum over $R_{t+1}^{s+1}$
\begin{align}
 &\sum\limits_{t=0}^{m-1} R_{t+1}^{s+1} \leq  \sum\limits_{t=0}^{m-1} \biggl[ \mathbb{E}\left[ f(x_t^{s+1}) \right]  - \frac{\eta}{2} \mathbb{E}\left[  ||\nabla f(x_t^{s+1})||^2 \right]  + \frac{L^2 \Delta \eta^3}{2b}\sum\limits_{i \in I_t}  \sum\limits_{j=t-\tau_i}^{t-1} \mathbb{E}\left[ ||v_j^{s+1}||^2 \right] \nonumber \\
&+ (\frac{\eta^2L}{2} + c_{t+1}\eta^2)  \mathbb{E}\left[ ||v_{t}^{s+1}||^2 \right]   +  c_{t+1}(1+\eta\beta_t)  \mathbb{E} \left[ ||x_t^{s+1} - \tilde{x}^s||^2 \right] \biggr] \nonumber \\
& \leq  \sum\limits_{t=0}^{m-1} \biggl[ \mathbb{E}\left[ f(x_t^{s+1}) \right]  - \frac{\eta}{2} \mathbb{E}\left[  ||\nabla f(x_t^{s+1})||^2 \right]   +  c_{t+1}(1+\eta\beta_t)  \mathbb{E} \left[ ||x_t^{s+1} - \tilde{x}^s||^2 \right] \nonumber   \\
& + (\frac{L^2 \Delta^2 \eta^3}{2} +  \frac{\eta^2L}{2} + c_{t+1}\eta^2)  \mathbb{E}\left[ ||v_{t}^{s+1}||^2 \right]   \biggr] \nonumber \\
 & \leq  \sum\limits_{t=0}^{m-1} \biggl[ \mathbb{E}\left[ f(x_t^{s+1}) \right]  - \frac{\eta}{2} \mathbb{E}\left[  ||\nabla f(x_t^{s+1})||^2 \right] +  c_{t+1}(1+\eta\beta_t)  \mathbb{E} \left[ ||x_t^{s+1} - \tilde{x}^s||^2 \right]  \nonumber   \\
&   +  \frac{2}{1-2L^2 \Delta^2 \eta^2}(\frac{L^2 \Delta^2 \eta^3}{2} + \frac{\eta^2L}{2} + c_{t+1}\eta^2)   \mathbb{E}\left[ ||u_{t}^{s+1}||^2 \right] \biggr]  \nonumber \\
& = \sum\limits_{t=0}^{m-1} R_{t}^{s+1} - \sum\limits_{t=0}^{m-1} \left[  \Gamma _t   \mathbb{E}\left[  ||\nabla f(x_t^{s+1})||^2 \right] \right]
\end{align}
where the last inequality follows the upper bound of $v_t^{s+1}$ in \cite{reddi2015variance}, and we define
\begin{eqnarray}
c_t =  c_{t+1}\left(1+\eta\beta_t + \frac{4L^2\eta^2}{(1-2L^2 \Delta^2 \eta^2)b}\right) + \frac{4L^2}{(1-2L^2 \Delta^2 \eta^2)b}  \left(\frac{L^2 \Delta^2 \eta^3}{2} + \frac{\eta^2L}{2}\right)
\end{eqnarray}

\begin{eqnarray}
 \Gamma _t = \frac{\eta}{2} -\frac{4}{(1-2L^2 \Delta^2 \eta^2)}  (\frac{L^2 \Delta^2 \eta^3}{2} + \frac{\eta^2L}{2} + c_{t+1}\eta^2)
\end{eqnarray}

 We set $c_{m} = 0$, and $\tilde{x}^{s+1} = x^{s+1}_{m}$, and $\gamma = \min\limits_t \Gamma_t$,  thus $R_{m}^{s+1} =  \mathbb{E}\left[  f(x_m^{s+1}) \right] =  \mathbb{E}\left[  f(\tilde x^{s+1}) \right]$, and
$R_{0}^{s+1} =  \mathbb{E}\left[  f(x_0^{s+1}) \right] =  \mathbb{E}\left[  f(\tilde x^{s}) \right]$. Summing up all epochs, the following inequality holds,
\begin{eqnarray}
\frac{1}{T}\sum\limits_{s=0}^{S-1}\sum\limits_{t=0}^{m-1} \mathbb{E}\left[  ||\nabla f(x_t^{s+1})||^2 \right]  \leq \frac{\mathbb{E}\left[  f( x^{0})  -  f( x^{*}) \right] }{T\gamma}
\end{eqnarray}

\end{proof}

\section{Proof of Theorem \ref{thm_m_2}}
\begin{proof}[Proof of Theorem \ref{thm_m_2}]
Following the proof of Theorem \ref{thm_m_1}, we let $c_m=0$, $\eta_t = \eta = \frac{u_0b}{Ln^\alpha}$, $\beta_t = \beta = 2L$,
$0<u_0<1$, and $0<\alpha<1$. We define $\theta$, and get its upper bound,

\begin{eqnarray}
\theta &=& \eta \beta + \frac{4L^2\eta^2}{(1-2L^2\Delta^2\eta^2)b} \nonumber \\
&=&  \frac{2u_0b}{n^{\alpha}} + \frac{4u_0^2b}{n^{2\alpha} - 2\Delta^2u_0^2b^2} \nonumber \\
&\leq&   \frac{6u_0b} {n^{\alpha}}
\end{eqnarray}
where we assume $n^{2\alpha} - 2\Delta^2u_0^2b^2 \geq n^{\alpha}$.
We set $m = \lfloor  \frac{n^{{\alpha}}}{6u_0b}  \rfloor$, from
the recurrence formula between $c_t$ and $c_{t+1}$, $c_0$ is upper bounded,
\begin{eqnarray}
c_0& =& \frac{2L^2}{(1-2L^2\Delta^2\eta^2)b} \left( L^2\Delta^2\eta^3 + \eta^2L \right) \frac{(1+\theta)^m - 1}{\theta} \nonumber \\
&\leq& \frac{2L \left( \frac{u_0^3\Delta^2b^3}{n^{3\alpha}} + \frac{u_0^2b^2}{n^{2\alpha}}  \right)}{ \left(1 - 2 L^2 \Delta^2\eta^2  \right)  \left(\frac{2u_0b^2}{n^{\alpha}} + \frac{4u_0^2b^2}{n^{2\alpha} - 2\Delta^2u_0^2b^2}   \right)} \left( (1+\theta)^m -1 \right) \nonumber \\
&\leq& \frac{L(u_0b\Delta^2+1)}{3} \left( (1+\theta)^m -1 \right) \nonumber \\
&\leq & \frac{L(u_0b\Delta^2+1)}{3} (e-1)
\end{eqnarray}
where the final inequality follows from that $(1+\frac{1}{l})^l$ is increasing for $l>0$, and $\lim\limits_{l\rightarrow  \infty}(1 + \frac{1}{l})^l = e$. From Theorem \ref{thm_m_1}, we know that $c_0 < \frac{\beta}{2} = L$, then $u_0b\Delta^2 < \frac{1}{2}$.
 $c_t$ is decreasing with respect to $t$, and $c_0$ is also upper bounded. Now, we can get a lower bound of $\gamma$,

\begin{eqnarray}
\gamma&=& \min_t \Gamma_t \nonumber \\
&\geq &   \frac{\eta}{2} -\frac{4}{(1-2L^2 \Delta^2 \eta^2)}  (\frac{L^2 \Delta^2 \eta^3}{2} + \frac{\eta^2L}{2} + c_{0}\eta^2)   \nonumber \\
&\geq& \frac{\eta}{2} - 4n^{\alpha}(\frac{L^2 \Delta^2 \eta^3}{2} + \frac{\eta^2L}{2} + c_{0}\eta^2) \nonumber \\
&\geq & (\frac{1}{2} - \frac{14\Delta^2u_0^2b^2 + 14u_0b}{3})\eta \nonumber \\
&\geq& \frac{\sigma b}{Ln^{\alpha}}
\end{eqnarray}
There exists a small value $\sigma$ that the final inequality holds if
$\frac{1}{2} > \frac{14\Delta^2u_0^2b^2 + 14u_0b}{3}$.
So, if $\Delta^2$ has an upper bound $ \Delta^2 < \min \{ \frac{1}{2u_0b}, \frac{3 - 28u_0b}{28u_0^2b^2}  \}
$ , we can prove the final conclusion,
\begin{eqnarray}
\frac{1}{T}\sum\limits_{s=0}^{S-1}\sum\limits_{t=0}^{m-1} \mathbb{E}\left[  ||\nabla f(x_t^{s+1})||^2 \right]  \leq \frac{Ln^{\alpha}\mathbb{E}\left[  f(\tilde x^{0})  -  f(\tilde x^{*}) \right] }{b T \sigma}
\end{eqnarray}

\end{proof}

\end{document}